\documentclass[11pt,a4paper]{article}
\usepackage{mdframed}
\usepackage[dvipsnames]{xcolor} %jliu
\usepackage{algorithmic}
\usepackage{xcolor}
\usepackage{etoolbox}
\usepackage[final]{acl}

\usepackage{times}
\usepackage{latexsym}

\usepackage[T1]{fontenc}
\usepackage[utf8]{inputenc}
\usepackage{amsmath}
\usepackage{amsfonts} %jliu
\usepackage{booktabs} %jliu
\usepackage{multirow}
\usepackage{pifont}

\usepackage{microtype}

\usepackage{inconsolata}

\usepackage{graphicx}
\usepackage{algorithm}
\usepackage{enumitem}

\title{RCR-Router: Role-Conditioned Context Routing for Resource-Efficient Multi-Agent LLM Systems}

\author{
  \textbf{Jun Liu\textsuperscript{1}},
  \textbf{Pu Zhao\textsuperscript{2}},
  \textbf{Zhenglun Kong\textsuperscript{3}},
  \textbf{Peiyan Dong\textsuperscript{4}},
  \textbf{Fan Yang\textsuperscript{5}},
  \textbf{Changdi Yang\textsuperscript{2}},\\
  \textbf{ Wuzhe Fan\textsuperscript{6}},
  \textbf{Tianqin Li \textsuperscript{1}},
  \textbf{Hao Tang\textsuperscript{1,7}},
  \textbf{Geng Yuan\textsuperscript{8}},
  \textbf{Wei Niu\textsuperscript{8}},
  \textbf{Wenbin Zhang\textsuperscript{9}}, \\
  \textbf{ Joannah Nanjekye\textsuperscript{1}},
  \textbf{Xue Lin\textsuperscript{2}},
  \textbf{Dong Huang\textsuperscript{1}},
  \textbf{Yanzhi Wang\textsuperscript{2}},
\\
\\
  \textsuperscript{1}Carnegie Mellon University,
  \textsuperscript{2}Northeastern University, \textsuperscript{3}Harvard University,
  \textsuperscript{4}MIT,
  \textsuperscript{5}Fujitsu,
\\
  \textsuperscript{6}McGill University,
  \textsuperscript{7}National University of Singapore,
  \textsuperscript{8}University of Georgia, \\
  \textsuperscript{9}Florida International University,
 \\ 
}

\begin{document}
\maketitle
\begin{abstract}
Large language model (LLM) agents are increasingly organized into multi-agent systems for complex reasoning and decision-making tasks.
However, scaling such systems can substantially increase inference cost, especially when multiple agents repeatedly process overlapping or role-irrelevant context.
Many existing multi-agent frameworks focus on orchestration and role specialization, but still rely on full-context sharing or static context routing, leading to redundant token consumption, increased runtime, and inefficient memory utilization.
We present \emph{Role-Conditioned Context Routing (RCR)}, an execution-time layer that allocates role-relevant shared-memory items to agents under token budgets.
Given a structured shared memory and a set of role-specialized agents, RCR selectively routes task-relevant memory items to each agent according to its role, optional task-stage metadata, and token budget.
RCR integrates role-conditioned relevance scoring, token-budgeted memory selection, and iterative memory updates to reduce redundant context exposure during multi-agent execution.
Experiments on multi-hop QA, embodied reasoning (ALFWorld), and web-based decision-making (WebShop) show that RCR reduces token usage by 35\% on average and runtime by 38\% on average relative to Full-Context, while maintaining or improving task performance.
\end{abstract}

\begin{figure}[t]
\centering
\includegraphics[width=0.98\linewidth]{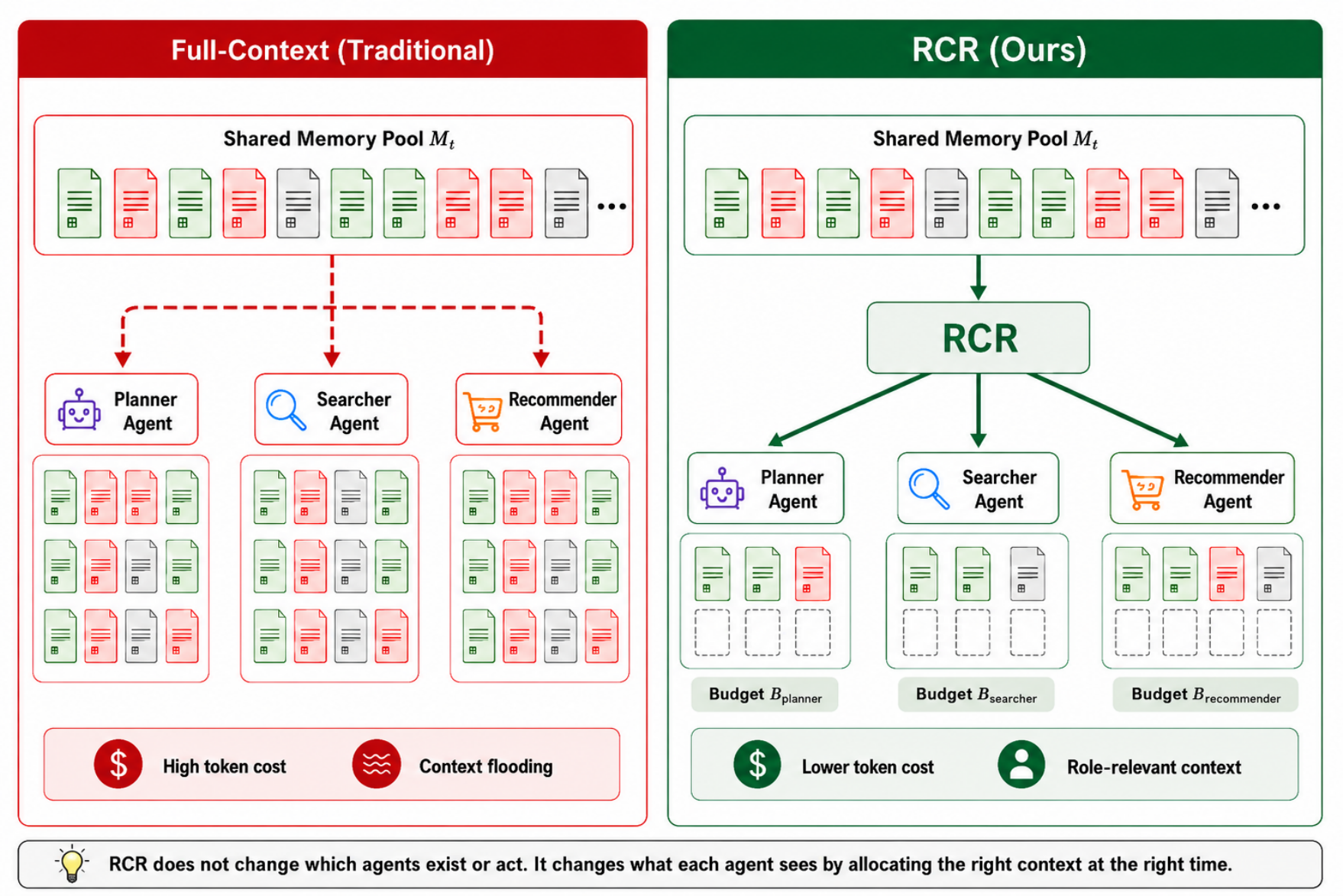}
\caption{
Motivation of RCR.
Traditional Agent (Full-Context) exposes all agents to the entire shared memory, leading to
redundant token usage and context flooding.
RCR keeps the agent topology unchanged and allocates compact, role-relevant
memory subsets under token budgets, changing what each agent sees rather than which agents are used.
}
\label{fig:motivation}
%\vspace{-15pt}
\end{figure}

\section{Introduction}

%As LLM-based agents~\cite{touvron2023llama,vicuna2023} become more capable and widely deployed, the execution cost of multi-agent systems becomes increasingly significant. Yet existing frameworks primarily focus on orchestration, role specialization, and tool use, while often overlooking token inefficiency caused by redundant context exposure. 
%\textcolor{red}{(the logic seems not consistent here. The above sentence mentions token inefficiency, the next sentence discusses role paradigm and modular coordination. these two seem not related.)}
%Recent multi-agent LLM frameworks such as \emph{CrewAI}~\cite{crewai2024}, \emph{Microsoft AutoGen}~\cite{wu2023autogen}, \emph{LangChain/LangGraph}~\cite{langgraph2025}, OpenAI \emph{Swarm}~\cite{openai2024swarm}, and Claude-style agent systems~\cite{anthropic} adopt a role-based paradigm, where each agent specializes in a distinct function, such as planning, searching, execution, or verification.This paradigm enables modular coordination and has become a common design pattern for scalable LLM systems.

As LLM-based agents~\cite{touvron2023llama,liu2026structured,vicuna2023} become more capable and widely deployed, the execution cost of multi-agent systems becomes increasingly significant.
Recent multi-agent LLM frameworks such as \emph{CrewAI}~\cite{crewai2024}, \emph{Microsoft AutoGen}~\cite{wu2023autogen}, \emph{LangChain/LangGraph}~\cite{langgraph2025}, OpenAI \emph{Swarm}~\cite{openai2024swarm}, and Claude-style agent systems~\cite{anthropic} adopt a role-based paradigm, where agents specialize in functions such as planning, searching, execution, or verification.
While this paradigm enables modular coordination, it also makes context allocation critical: different roles require different information at different stages.
However, existing frameworks primarily specify how agents are orchestrated, while often relying on full-context sharing, fixed templates, or static routing.
As a result, multiple agents repeatedly process overlapping or role-irrelevant context, causing redundant token consumption, increased runtime, and inefficient memory utilization.

Despite such role specialization, existing systems often rely on \emph{uniform context exposure}: agents receive either fixed templates~~\cite{langgraph2025}, static routing, or the full interaction history.
As a result, multiple agents may repeatedly process overlapping or role-irrelevant context during execution, leading to unnecessary token usage, higher runtime, and possible distraction from task-relevant information.
Retrieval-based methods such as \emph{Self-RAG}~\cite{asai2024self} and \emph{End-to-End Beam Retrieval}~\cite{zhang2023end} improve evidence selection, but they typically retrieve a single shared context rather than allocating different context subsets to different agents.
Similarly, \emph{Multi-Agent Debate}~\cite{du2023improving} improves coordination through critique, but does not explicitly solve token-budgeted, role-conditioned context allocation over shared memory.

This mismatch becomes critical as task complexity grows.
Different agents require different information at different stages:
for example, in multi-hop QA, a \emph{Planner} focuses on reasoning~\cite{sun2026beyond} state, a \emph{Searcher} on evidence retrieval, and a \emph{Synthesizer} on answer aggregation~\cite{langgraph2025}.
Providing identical context to all agents wastes tokens~\cite{kong2026ai,tan2025perturbation,li2025mutual} and can introduce role-irrelevant information~\cite{ji2025computation,10.1145/3747842,liu2025toward}, which may degrade reasoning quality~\cite{yang2025fairsmoe,liu2024tsla,zhang2026roots,zhao2025efficient,tan2025end}.
This motivates our central premise: multi-agent LLM systems require
execution-time context allocation that is both role-conditioned and
resource-aware.
Figure~\ref{fig:motivation} illustrates this contrast.

To address this problem, we present \emph{Role-Conditioned Context Routing (RCR)}, an execution-time layer that allocates role-relevant shared-memory items to agents under token budgets.
Rather than introducing a new agent orchestration architecture, RCR focuses on context exposure: given an existing multi-agent structure and
shared memory, it selects which information each agent should receive under
token-budget constraints. 
As shown in Figure~\ref{fig:motivation}, Full-Context broadcasts the same memory pool to every agent, whereas RCR allocates compact, role-relevant memory subsets within a fixed role-based topology.

Given a structured shared memory, RCR dynamically routes relevant subsets to each agent based on its role, optional task-stage metadata, and token budget.
Our method consists of three components:
(1) a reward-optimized token budget allocator that selects per-agent context capacity,
(2) an importance scorer that computes role- and stage-conditioned relevance,
and (3) a size-aware routing mechanism that selects memory items under budget constraints.
We further use an \emph{Iterative Routing Feedback (IRF)} process, where agent outputs are incorporated into shared memory to support iterative context refinement across interaction rounds.

We report main results on HotPotQA, ALFWorld, and WebShop, with additional
multi-hop QA results on MuSiQue and 2WikiMultihop provided in
Table~\ref{tab:multiqa_results} of Appendix~\ref{app:experiment}.
Overall, RCR provides a favorable efficiency--performance trade-off, reducing token usage by 35.1\%  and runtime by 38.0\% on average relative to Full-Context, while maintaining or improving task performance.

Our contributions are as follows:
\iffalse

\begin{itemize}
    \item We identify redundant context exposure as a key efficiency bottleneck in multi-agent LLM systems, where full-context sharing and static routing cause unnecessary token usage and role-irrelevant information.

    \item We formulate execution-time shared-memory routing as a role- and stage-conditioned budgeted selection problem, focusing on which context each agent should receive under token constraints.

    \item We implement RCR as an execution-time context allocation layer that combines reward-optimized budget allocation, role-conditioned relevance scoring, size-aware memory selection, and iterative memory updates.

    \item We evaluate RCR across multi-hop QA, embodied reasoning, and web navigation benchmarks, showing improvements in the efficiency--performance trade-off over full-context, static-routing, retrieval-based, and debate-based baselines.
\end{itemize}
\fi

\begin{itemize}[leftmargin=*]
    \item We identify and formalize {multi-agent role-conditioned context routing} as a distinct execution-time resource-allocation problem: given a shared memory and fixed role-specialized agents, deciding which memory items should be exposed to which agent under token-budget constraints.
    \item We instantiate this formulation as RCR, a
    modular routing layer that integrates role- and
    stage-conditioned relevance scoring, adaptive budget
    allocation, and size-aware budgeted selection
    without changing the underlying agent topology.
    \item Under controlled conditions with matched
    agents, prompts, memory pools, tools, call limits,
    and stopping rules, we show across multi-hop QA,
    embodied reasoning, and web navigation benchmarks
    that role-conditioned context allocation improves
    the efficiency--performance trade-off over
    full-context, static-routing, retrieval-based,
    and debate-based baselines.
\end{itemize}

\section{Related Work}
\noindent\textbf{Multi-Agent LLM Systems.}
Recent multi-agent LLM frameworks (e.g., AutoGen~\cite{wu2023autogen},
LangGraph~\cite{langgraph2025}, CrewAI~\cite{crewai2024},
Swarm~\cite{openai2024swarm}) enable flexible role-based agent orchestration
but largely rely on uniform or template-based context exposure.
\textit{LLM-MAS}~\cite{han2024llm} surveys key challenges including task
allocation and memory management, but does not propose a token-budgeted
routing mechanism. \textit{MasRouter}~\cite{yue2025masrouter} addresses
system-level routing and model selection, while
\textit{AgentNet}~\cite{yang2025agentnet} focuses on decentralized
coordination and evolving agent topology. \textit{SupervisorAgent}~\cite{lin2026stop} reduces token consumption through runtime supervision. These approaches are orthogonal to RCR's pre-call context allocation and could be combined. 

\noindent\textbf{Memory Management for LLM Agents.}
\textit{A-Mem}~\cite{xu2025mem} organizes memories through
indexing and linking, while \textit{G-Memory}~\cite{zhang2025g-memory}, H-MEM~\cite{sun-etal-2026-h}
structures multi-agent histories with hierarchical memory and trajectory
tracing. \textit{HiAgent}~\cite{hu2024hiagent} manages working memory
through subgoal-based chunks, and \textit{Self-Notes}~\cite{lanchantin2305learning}
enhance reasoning with intermediate notes. \textit{IMA}~\cite{yuen2025intrinsic}
maintains role-aligned agent-specific memories, while
\textit{Collaborative Memory}~\cite{rezazadeh2025collaborative}, MIRIX~\cite{wang2025mirix} focuses on
permission-aware memory sharing. These methods~\cite{liu2026palm,shen2026fine,fu2025learning,li2026toward} primarily address how memories are stored~\cite{yu2025core3d,fu2025learning,li2026toward}, organized~\cite{liu2026consensus,cheng2026toward}, or accessed~\cite{du2026search,11632928}. In contrast, RCR targets
execution-time context allocation: given a shared memory~\cite{10889613}, it selectively
routes role-relevant subsets to different agents under token-budget
constraints. This dimension is orthogonal to memory organization.

\section{The Proposed Method}

\subsection{Problem Formulation}
We consider a multi-agent LLM system composed of $N$ collaborative agents $\mathcal{A} = \{ A_1, A_2, \dots, A_N \}$ interacting over a shared task. Each agent $A_i$ operates with a specific \textit{role} $R_i$ (e.g., Planner,  Searcher, Recommender) and interacts with other agents and external tools in discrete \textit{interaction rounds} $t = 1, 2, \dots, T$.

At each round $t$, agents exchange messages and perform reasoning based on a
{Shared Memory Store} $M_t$, which contains:
\iffalse
\begin{itemize}
    \item {Agent interaction history:} prior communications and intermediate reasoning traces;
    \item {Task-relevant knowledge:} retrieved documents, external facts, or tool outputs;
    \item {Structured state representations:} entities, subplans, and tool traces encoded in structured formats %(e.g., YAML, graphs, tables).
\end{itemize}
\fi
%The shared memory store $M_t$ contains:
(I) \textit{agent interaction history}, including prior communications and
intermediate reasoning traces;
(II) \textit{task-relevant knowledge}, including retrieved documents, external facts, or tool outputs; and
(III) \textit{structured state representations}, including entities, subplans, and tool traces encoded in structured formats.

Each agent $A_i$ receives a \textit{routed context} $C_t^i \subseteq M_t$,
selected by the \textbf{learned routing policy} according to the agent’s
role $R_i$, current task stage $S_t$, and the token budget $B_i$ predicted
by the Token Budget Allocator. 
This design incorporates a mechanism for context selection and allocation. %end-to-end differentiable 

Each agent $A_i$ then performs an \textit{LLM Query} based on its routed context $C_t^i$:
\begin{equation}
LLM\_output^i_t = LLM( \text{Prompt}(C_t^i) ),
\end{equation}

The global objective of the system is to maximize collaborative task success while minimizing cumulative token consumption across all agents and rounds.

We formulate context routing as a budgeted relevance maximization problem.
For agent $A_i$ at round $t$, RCR selects a subset of shared-memory items
$C_t^i \subseteq M_t$:
\begin{equation}
C_t^i=\arg\max_{C'\subseteq M_t:\,\sum_{m\in C'}\ell(m)\le B_i}
\sum_{m\in C'} \alpha(m;R_i,S_t).
\label{eq:routing_objective}
\end{equation}
where $\alpha(m;R_i,S_t)$ is the role- and stage-conditioned relevance score,
$\ell(m)$ is the token length of memory item $m$, and $B_i$ is the allocated
token budget.

\iffalse  %do not need greedy
Eq.~(\ref{eq:routing_objective}) serves as an execution-time selection
interface; the main modeling challenge lies in estimating the role-conditioned utility $\alpha(m;R_i,S_t)$. Solving Eq.~(\ref{eq:routing_objective}) exactly is a 0--1 knapsack-style
budgeted subset selection problem. During inference, RCR uses a deterministic
size-aware greedy heuristic with a best-singleton fallback. Specifically, it first ranks memory items by relevance density
$\alpha(m;R_i,S_t)/\ell(m)$ and greedily adds items to $C_t^i$ until the
budget is exhausted. It then compares the resulting greedy set with the
highest-scoring single memory item that fits within the budget and returns the
higher-utility solution \textcolor{red}{(maybe define $C_{\mathrm{greedy}}$ and $C_{\mathrm{single}}$)}:
\[
C_t^i =
\arg\max_{C \in \{C_{\mathrm{greedy}}, C_{\mathrm{single}}\}}
\sum_{m\in C}\alpha(m;R_i,S_t).
\]
Here, $\alpha(m;R_i,S_t)/\ell(m)$ is used only as a greedy ranking criterion,
not as the formal utility objective. The formal utility remains the additive
relevance objective in Eq.~(\ref{eq:routing_objective}).
\leftrightarrow
\fi  %%%%%%%%%%%%%%%%%end

Eq.~(\ref{eq:routing_objective}) defines the execution-time routing interface:
given role-conditioned utilities and a token budget, RCR selects a compact
memory subset for each agent. The main modeling challenge is estimating
$\alpha(m;R_i,S_t)$ so that selected items match the agent's role and current
stage. At inference time, we use a lightweight size-aware online selector:
memory items are ordered by normalized relevance
$\alpha(m;R_i,S_t)/\ell(m)$, a feasible set is constructed under $B_i$, and the
result is compared with the best feasible singleton. The candidate with higher
additive relevance is returned:
\[
C_t^i =
\arg\max_{C \in \{C_{\mathrm{set}}, C_{\mathrm{single}}\}}
\sum_{m\in C}\alpha(m;R_i,S_t).
\]
The normalized score is used only for size-aware ordering; the utility remains
the additive objective in Eq.~(\ref{eq:routing_objective}).

\iffalse %%%%%%%%%%%%%%not need it 
\noindent\textbf{Approximation property.}
Given $\alpha(m;R_i,S_t)$, Eq.~\eqref{eq:routing_objective} is a standard
budgeted selector; density-greedy with a best-singleton fallback gives the
standard $1/2$--approximation. We use it as a lightweight online selector,
while RCR's key challenge is learning reliable utilities and budgets.
\fi %%%%%%%%%%%%%%%%%%%%%%%%

\subsection{Role-Conditioned Context Routing with Semantic Abstraction}
\label{sec:role-aware} 
\noindent\textbf{Context Routing Framework.}
We introduce \textbf{RCR}, a modular context-routing layer that learns
\emph{what information should be routed to which role} under token-budget
constraints. Rather than defining a new communication protocol, RCR coordinates agents through role- and stage-conditioned routing over a shared memory store. Figure~\ref{fig:RCR-arch} illustrates the architecture. Shared memory $M$ stores agent history, external knowledge, and task-relevant entities in structured records with role tags, stage tags, timestamps, and content fields, allowing indexing and
semantic filtering.

\begin{figure*}[t]
    \centering
    \includegraphics[width=0.90\textwidth]{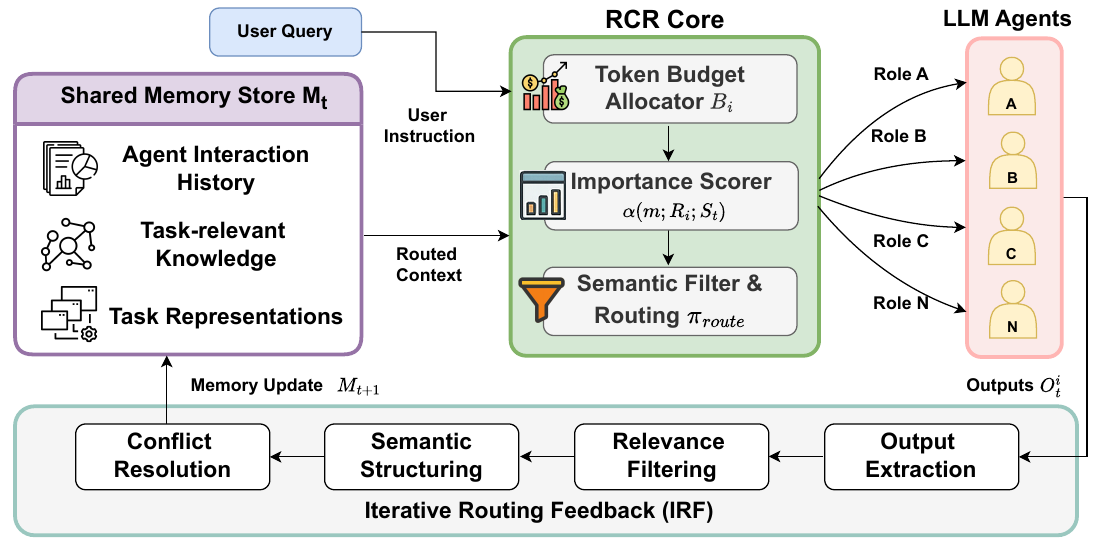}
            %\vspace{-5pt}
    \caption{
\textbf{Iterative RCR architecture with adaptive feedback loop.}
At each interaction round $t$, RCR dynamically routes semantically filtered memory 
to each agent based on its role and current task stage. 
Agents process the routed contexts and produce structured outputs, 
which are integrated into the updated Shared Memory $M_{t+1}$ through the Memory Update step. 
This iterative process forms a closed-loop cycle 
($M_t \!\rightarrow\!$ Role-Aware Routing $\!\rightarrow\!$ Agents $\!\rightarrow\!$ IRF $\!\rightarrow\! M_{t+1}$), 
enabling progressive refinement of contextual relevance and efficient multi-agent coordination.
}
 %\textcolor{red}{(update)}
    \label{fig:RCR-arch}
%\vspace{-10pt}
\end{figure*}

%Figure~\ref{fig:RCR-arch} illustrates the overall architecture of RCR. The system operates on a Shared Memory Store $M$, which encodes historical agent interactions, external knowledge, and task-relevant entities in structured formats such as YAML, graphs, or tables. This abstraction facilitates efficient memory indexing and semantic filtering.

\noindent\textbf{RCR Core.}
Multi-agent LLM systems require role-specific access to shared memory under strict context constraints.
The RCR framework introduces an importance-based structured routing policy that assigns token-efficient, role-relevant memory context to each agent at every interaction round. 
It consists of three modular components: (1) Token Budget Allocator, (2) Importance Scorer, and (3) Semantic Filter with Routing Logic.

\subsubsection{Reward-Optimized Token Budget Allocator}

We formulate token budget allocation as a resource-aware policy optimization problem.
At each interaction round $t$, the allocator observes the role embedding $r_i$,
stage embedding $s_t$, and a usage vector $u_i$ summarizing recent token
consumption, context utilization, latency, and task feedback. It then samples
a token budget for agent $A_i$:
\begin{equation}
B_i \sim \pi_\theta(\cdot \mid r_i, s_t, u_i),
\end{equation}
where $\pi_\theta$ is a lightweight policy network over a discrete budget set
$\mathcal{B}=\{512,1024,2048,4096\}$.

After completing a multi-agent trajectory $\tau$, the allocator receives an
episode-level reward:
\iffalse
\begin{equation}
\label{eq:reward_function}
R(\tau)=Q(\tau)-\lambda_T \mathrm{Tok}(\tau)-\lambda_L \mathrm{Lat}(\tau),
\end{equation}
\fi
\begin{equation}
\label{eq:reward_function}
R(\tau)
=
Q(\tau)
-
\lambda_{\mathrm{Tok}} \, \mathrm{Tok}(\tau)
-
\lambda_{\mathrm{Lat}} \, \mathrm{Lat}(\tau),
%\label{eq:reward}
\end{equation}
where $Q(\tau)$ denotes the benchmark-specific primary metric
(F1 for HotPotQA, success rate for ALFWorld, and reward score for WebShop),
$\mathrm{Tok}(\tau)$ denotes total token usage, and $\mathrm{Lat}(\tau)$ denotes end-to-end latency. To make reward terms comparable across benchmarks, token and latency costs are normalized by development-set statistics before applying $\lambda_{\mathrm{Tok}}$ and $\lambda_{\mathrm{Lat}}$, which are tuned only on the development split. 
Sensitivity to $\lambda$ values is shown in Appendix~\ref{app:lambda_sensitivity}.

We optimize $\pi_\theta$ using a policy-gradient objective with a moving-average
baseline for variance reduction and entropy regularization for exploration.
During inference, we select the budget with the highest policy probability.
Given $B_i$, RCR orders memory items by normalized relevance
$\alpha(m;R_i,S_t)/\ell(m)$ and constructs a budget-feasible routed context
using the size-aware online selector.
Details are provided in Appendix~\ref{app:rl_allocator}.
%Details of the PPO-style stabilized implementation are provided in Appendix~\ref{app:rl_allocator}.

\subsubsection{Importance Scorer}

The importance scorer is trained with AdamW ($5\times10^{-4}$ learning rate, $10^{-4}$ weight decay, batch size 256, cosine decay with 10\% warm-up). Because the LLM is black-box and routing is discrete, we use weak supervision from logged trajectories. Positive labels come from memory items linked to successful evidence chains, tool outputs, or action traces. Negatives are contradicted or same-role--stage hard-negative items; related unused items are down-weighted soft negatives, while ambiguous unused items are masked from the loss. The scorer is optimized with a weighted binary relevance loss conditioned on the role and task stage.
Appendix~\ref{app:scorer} gives the full loss, examples,and label statistics.
\noindent\textbf{Training and leakage control.}
Both routing modules are trained only from training-split trajectories collected under
the same roles, prompts, tools, memory pool, and stopping rules used in evaluation.
The budget allocator observes only online execution signals and is optimized with a
PPO-style clipped policy-gradient objective; it does not use gold answers, future
outcomes, or test metrics. The importance scorer uses conservative weak labels:
positives are linked evidence/tool/action records, negatives are contradicted or
same-role--stage hard negatives, and ambiguous unused items are discarded.
Appendix~\ref{app:rl_allocator} and Appendix~\ref{app:scorer} provide the full
objectives, logging assumptions, label statistics, and discard rates.

\subsubsection{Semantic Filter and Routing}
The routing policy in Eq.~(\ref{eq:routing_objective}) uses a budgeted
size-aware online selector: it filters feasible items
$M'=\{m\in M_t\mid \ell(m)\le B_i\}$, ranks them by
$\alpha(m;R_i,S_t)/\ell(m)$, forms $C_{\mathrm{set}}$ under $B_i$, and returns
the higher-utility candidate between $C_{\mathrm{set}}$ and
$C_{\mathrm{single}}=\{\arg\max_{m\in M'}\alpha(m;R_i,S_t)\}$.

The resulting $C_t^i$ is passed as the input context for the next LLM query of agent $A_i$. The relevance score is conditioned on the agent role and task stage, while the budget $B_i$ is determined by the reward-optimized allocator.
Algorithm~\ref{alg:rcr_context_routing} in the Appendix summarizes the routing workflow.
All experiments use this deterministic size-aware online selector for controlled comparison.

\subsection{Iterative Routing with Feedback}

To support tasks that evolve over multiple interaction rounds, RCR incorporates
\textbf{Iterative Routing Feedback (IRF)}. At round $t$, each agent $A_i$ receives
a budgeted context $C_t^i$ based on its role $R_i$, task stage $S_t$, and budget
$B_i$, then produces outputs such as messages, tool calls, retrieved evidence, or
intermediate reasoning steps.

IRF converts these outputs into structured memory items---including facts,
subplans, tool results, and evidence---and updates the shared memory from $M_t$
to $M_{t+1}$ for the next routing step, as shown in Figure~\ref{fig:RCR-arch}.
This allows later rounds to use updated context without repeatedly exposing agents
to stale or redundant memory. We treat IRF as an empirical refinement mechanism,
not a monotonic convergence guarantee.

\noindent\textbf{Memory Update.}
After each interaction round, RCR updates the shared memory through:
(I) \textbf{Output Extraction}, which parses each $O_t^i$ into
\texttt{(role, stage, key\_fact, evidence, timestamp)} and retains atomic
records corresponding to task-relevant facts, evidence statements, tool results, and reasoning summaries;
%(II) \textcolor{red}{(this part is very similar to section 3.2.3, what is the connection with $C_{\mathrm{g}}$, do we need to repeat this step here)} \textbf{Relevance Filtering}, which scores each candidate item by
%$\alpha(m;R_i,S_t)$ and selects items under $B_i$ using $C_t^i=\mathrm{GreedySelect}(M_t,\alpha(m;R_i,S_t)/\ell(m),B_i)$;
(II) \textbf{Relevance Filtering}, which scores newly extracted candidate
items by $\alpha(m;R_i,S_t)$, using the producing agent's role and current
task stage, and retains high-confidence records for insertion into
$M_{t+1}$;
and (III) \textbf{Semantic Structuring}, which stores retained items as
$m=\{\text{role\_tag},\text{stage\_tag},\text{timestamp},\text{content}\}$.
Deduplication is applied uniformly before routing and removes only
near-duplicate records within the same role--stage group; it is not tuned per method.

\noindent\textbf{Memory control.}
All controlled methods use the same atomic memory items, serialization,
deduplication, and conflict-resolution rules; only the routing policy differs.
Thus, token savings are not attributable to method-specific memory compression.
\iffalse
    \item \textbf{Conflict Resolution.}
    Two memory items $m_a$ and $m_b$ are treated as duplicates if they share the
    same role--stage tag and their textual contents are highly similar:
    \begin{equation}
    \mathrm{sim}(m_a,m_b)
    =
    \cos\!\big(\mathrm{emb}(m_a),\,\mathrm{emb}(m_b)\big),
    \label{eq:dup_sim}
    \end{equation}
    where $\mathrm{emb}(\cdot)$ is obtained from
    \texttt{text-embedding-3-small}~\cite{openai2024textembedding3}.
    If $\mathrm{sim}(m_a,m_b)>0.9$, we retain the more recent item:
    \begin{equation}
    m_{\text{keep}}
    =
    \arg\max_{m\in\{m_a,m_b\}}
    \mathrm{timestamp}(m).
    \label{eq:dup_keep}
    \end{equation}
%\end{enumerate}

Let $\Delta M_t$ denote the set of new structured memory items extracted from
agent outputs at round $t$. The shared memory is updated as:
\begin{equation}
M_{t+1}
=
\mathrm{Dedup}
\left(
M_t \cup \Delta M_t
\right),
\label{eq:memory_update}
\end{equation}
where $\mathrm{Dedup}(\cdot)$ applies the role--stage constrained duplicate
filtering described above. This process reduces redundant memory growth while
preserving recent task-relevant information.
\fi 
Appendix~\ref{appendix:update_example} provides the schema, prompts, and examples
used by a structured-output Memory Update Agent to convert raw agent outputs
into structured memory records. Following standard interfaces such as the
OpenAI Agents SDK~\cite{openai_agents_sdk} and Google ADK~\cite{google_adk_llm_agents},
RCR does not rely on task-specific hand-written parsers; the routing layer only
consumes structured memory records and is agnostic to the underlying agent
framework.

RCR therefore combines two processes:
(1) \emph{role-conditioned context routing}, where each agent receives a
budgeted context based on role, stage, and resource constraints; and
(2) \emph{memory update}, where structured agent outputs form the memory used for subsequent routing rounds.

\begin{table*}[t]
\centering
\caption{%
   Main results on HotPotQA, ALFWorld, and WebShop.
  Each benchmark is evaluated with its \emph{official} primary metric:
  F1 score for HotPotQA, task success rate (\%) for ALFWorld,
  and reward score (0--1) for WebShop.
  HotPotQA uses three roles, while
  ALFWorld and WebShop use four-role configurations as detailed in
  Appendix~\ref{app:role_config}.
  Runtime and token usage are mean values over five independent runs;
  primary metrics are mean\,$\pm$\,std.
  Full Precision/Recall breakdowns for HotPotQA are provided in
  Appendix~\ref{app:multi-hop_QA}. 
  %\textcolor{red}{(there are too many tables in the same page. It does not look nice. The other pages with texts only may look a bit boring. better to distribute tables in multiple pages.)}
}
\label{tab:combined_results}
%\setlength{\tabcolsep}{8pt}
%\resizebox{\linewidth}{!}
\resizebox{0.85\linewidth}{!}{%
\begin{tabular}{llccc}
\toprule
\textbf{Benchmark}
  & \textbf{Method}
  & \textbf{Avg Runtime (s)}$\downarrow$
  & \textbf{Token (K)}$\downarrow$
  & \textbf{Primary Metric}$\uparrow$ \\
\midrule
%% ── HotPotQA  (F1, 0–100) ─────────────────────────────────────────
\multirow{6}{*}{\shortstack[l]{HotPotQA\\(F1, 0--100)}}
  & Full-Context         & 150.65 & 5.10 & 74.6\,{\scriptsize$\pm$0.6} \\
  & Static Routing       & 128.29 & 3.85 & 76.1\,{\scriptsize$\pm$0.6} \\
  & Self-RAG             & 118.44 & 3.92 & 78.4\,{\scriptsize$\pm$0.5} \\
  & End-to-End Retrieval & 110.12 & 3.80 & 79.0\,{\scriptsize$\pm$0.5} \\
  & Multi-Agent Debate   & 105.33 & 4.06 & 80.0\,{\scriptsize$\pm$0.4} \\
  & \textbf{RCR (Ours)}  & \textbf{91.58}  & \textbf{3.63}
                         & \textbf{83.6}\,{\scriptsize$\pm$\textbf{0.5}} \\
\midrule
%% ── ALFWorld  (Success Rate %) ────────────────────────────────────
\multirow{6}{*}{\shortstack[l]{ALFWorld\\(Success Rate, \%)}}
  & Full-Context         & 145.3  & 6.82 & 59.1\,{\scriptsize$\pm$0.8} \\
  & Static Routing       & 122.8  & 5.12 & 62.2\,{\scriptsize$\pm$0.7} \\
  & Self-RAG             & 117.6  & 5.40 & 63.7\,{\scriptsize$\pm$0.7} \\
  & End-to-End Retrieval & 114.2  & 5.18 & 65.5\,{\scriptsize$\pm$0.6} \\
  & Multi-Agent Debate
                         & 136.7  & 5.95 & 60.2\,{\scriptsize$\pm$0.9} \\
  & \textbf{RCR (Ours)}  & \textbf{94.5}   & \textbf{4.30}
                         & \textbf{67.3}\,{\scriptsize$\pm$\textbf{0.8}} \\
\midrule
%% ── WebShop  (Reward, 0–1) ────────────────────────────────────────
\multirow{6}{*}{\shortstack[l]{WebShop\\(Reward, 0--1)}}
  & Full-Context         & 180.2  & 9.78 & 0.565\,{\scriptsize$\pm$0.009} \\
  & Static Routing       & 142.6  & 7.12 & 0.616\,{\scriptsize$\pm$0.008} \\
  & Self-RAG             & 134.3  & 6.85 & 0.642\,{\scriptsize$\pm$0.007} \\
  & End-to-End Retrieval & 120.7  & 6.47 & 0.639\,{\scriptsize$\pm$0.006} \\
  & Multi-Agent Debate   & 128.9  & 6.58 & 0.627\,{\scriptsize$\pm$0.007} \\
  & \textbf{RCR (Ours)}  & \textbf{108.7}  & \textbf{5.91}
                         & \textbf{0.653}\,{\scriptsize$\pm$\textbf{0.005}} \\
\bottomrule
\end{tabular}}
%\vspace{-8pt}
\end{table*}

\section{Experiments}
\label{sec:experiment} 

\subsection{Implementation Details}
\label{sec:ImplementationDetails}
\noindent\textbf{Model and Hardware Platform.}
All experiments were executed on a workstation with an Intel Xeon Silver CPU (32 cores), 128GB RAM, and one NVIDIA A6000 GPU.
We use OpenAI's \texttt{text-embedding-3-small} for memory-item embeddings
and retrieval features~\cite{openai2024textembedding3}. For agent LLM
backbones, we evaluate GPT-5-mini~\cite{openai_gpt5mini_2025},
Qwen3-Max~\cite{qwen3max_2025}, and Gemini 2.5 Pro~\cite{geminipro25_2025}.
 %~\footnote{\url{https://openai.com/index/GPT-5-mini-advancing-cost-efficient-intelligence/}}

\noindent\textbf{Benchmarks and Metrics.}
We report main results on HotPotQA~\cite{yang2018hotpotqa},
ALFWorld~\cite{ALFWorld20}, and WebShop~\cite{yao2022webshop},
covering multi-hop reasoning, embodied decision-making, and web navigation.
Additional multi-hop QA results on MuSiQue~\cite{trivedi2022musique}
and 2WikiMultiHop~\cite{xanh2020_2wikimultihop} are provided in
Appendix~\ref{app:experiment}. All reported results are averaged over 5 independent runs.

\noindent\textbf{Baselines.}
We compare RCR with \textbf{Full-Context Routing}~\cite{crewai2024},
\textbf{Static Routing}~\cite{langgraph2025},
\textbf{Self-RAG}~\cite{asai2024self},
\textbf{End-to-End Beam Retrieval}~\cite{zhang2023end}, and
\textbf{Multi-Agent Debate}~\cite{du2023improving}. Full-Context and
Multi-Agent Debate share $M_t$ without role-specific routing; Static Routing uses fixed role templates; Self-RAG and End-to-End Beam Retrieval are retrieval-based baselines using the same memory pool and comparable prompt budgets.

\noindent\textbf{Training.}
The two routing modules are trained on the training split and tuned on the validation
split, while all reported results are computed on the held-out test split. The allocator is optimized at the trajectory level for the efficiency--accuracy trade-off, whereas the scorer is trained at the item level for role-conditioned memory relevance. Full training objectives, logging assumptions, and weak-label statistics are provided in Appendix~\ref{app:rl_allocator} and Appendix~\ref{app:scorer}.
For more details, see Appendix~\ref{app:Implementation}.\\
\iffalse
\noindent\textbf{ Experimental Setup}
To ensure comparability, all controlled multi-agent methods use the same agent roles, role prompts, task prompts, tool interfaces, maximum interaction rounds, and stopping conditions within each benchmark. The only varied component is the context routing policy. Full-Context routes the entire shared memory to every agent, Static Routing uses fixed role-specific context templates, and Multi-Agent Debate uses the same roles and round limit but shares identical context among debating agents. Retrieval-based baselines use the same memory pool and token budget, but retrieve a single shared context rather than role-conditioned contexts.
\fi

\subsection{Results}
Table~\ref{tab:combined_results} reports the main results under the default
single-pass setting with fixed role configurations.
%Token reduction is our primary reproducible efficiency metric, while runtime is reported as an environment-dependent measure of end-to-end latency.
%In RCR, the trained Token Budget Allocator selects $B_i$ from $\{512,1024,2048,4096\}$ at inference time, while $B_i{=}2048$ is used only as a reference budget for normalizing baseline prompt budgets.
Each benchmark is evaluated with its official primary metric:
F1 for HotPotQA, success rate for ALFWorld, and reward score for WebShop.
RCR achieves a favorable efficiency--performance trade-off, reducing token
usage and runtime while preserving or improving the primary task metric.\\
\noindent\textbf{Backbone-Transfer Robustness.}
RCR transfers to Qwen3-Max and
Gemini 2.5 Pro without backbone-specific
retraining. Across all three benchmarks, it consistently reduces token usage
and runtime over Full-Context and Static Routing while preserving or improving
the primary task metric (Appendix~\ref{app:backbone_transfer}).

\subsection{Controlled Routing Comparison}
To isolate the routing policy, we compare RCR with alternatives using the same
agents, prompts, rounds, memory construction, tool interfaces, and stopping
criteria; only the memory selection policy differs.
Table~\ref{tab:controlled_routing} shows that RCR achieves the strongest
cost--quality trade-off among lexical, semantic, role-only, stage-only, and
fixed-budget routing variants under this controlled setup, suggesting that the
observed differences are not merely due to agent topology, prompt protocol, or
memory construction. Sensitivity to the token budget and routing iterations is reported in Appendix~\ref{Sensitivity_Analysis}.

\begin{table*}[t]
\centering
\small
\caption{
Controlled routing comparison under identical agents, prompts, interaction
rounds, memory construction, and stopping criteria. Only the memory selection
policy differs. We report mean $\pm$ standard deviation over five independent
runs. RCR improves the cost--quality trade-off over controlled routing variants;
gains over Fixed Budget are strongest on HotPotQA and WebShop, and moderate on
ALFWorld.
}
\label{tab:controlled_routing}
\resizebox{\linewidth}{!}{
\begin{tabular}{lccc|ccc|ccc}
\toprule
\multirow{2}{*}{\textbf{Method}}
& \multicolumn{3}{c|}{\textbf{HotPotQA}}
& \multicolumn{3}{c|}{\textbf{ALFWorld}}
& \multicolumn{3}{c}{\textbf{WebShop}} \\
& Tok.$\downarrow$ & Run.$\downarrow$ & F1$\uparrow$
& Tok.$\downarrow$ & Run.$\downarrow$ & Succ.$\uparrow$
& Tok.$\downarrow$ & Run.$\downarrow$ & Reward$\uparrow$ \\
\midrule
BM25 Retrieval
& $3.79{\pm}0.05$ & $103.4{\pm}1.8$ & $78.2{\pm}0.5$
& $4.44{\pm}0.06$ & $108.2{\pm}2.1$ & $63.1{\pm}0.7$
& $6.12{\pm}0.07$ & $123.6{\pm}2.4$ & $0.617{\pm}0.007$ \\

Semantic Retrieval
& $3.76{\pm}0.04$ & $99.7{\pm}1.6$ & $79.2{\pm}0.5$
& $4.40{\pm}0.05$ & $104.7{\pm}1.9$ & $64.0{\pm}0.6$
& $6.05{\pm}0.06$ & $119.4{\pm}2.1$ & $0.624{\pm}0.006$ \\

Role-only Routing
& $3.71{\pm}0.04$ & $96.8{\pm}1.5$ & $80.5{\pm}0.4$
& $4.36{\pm}0.05$ & $100.3{\pm}1.7$ & $65.9{\pm}0.6$
& $5.99{\pm}0.06$ & $115.6{\pm}1.9$ & $0.640{\pm}0.006$ \\

Stage-only Routing
& $3.74{\pm}0.04$ & $98.5{\pm}1.6$ & $79.7{\pm}0.5$
& $4.38{\pm}0.05$ & $101.4{\pm}1.8$ & $65.2{\pm}0.7$
& $6.01{\pm}0.06$ & $117.1{\pm}2.0$ & $0.633{\pm}0.006$ \\

Fixed Budget
& $3.69{\pm}0.03$ & $95.4{\pm}1.4$ & $81.2{\pm}0.5$
& $4.34{\pm}0.04$ & $98.9{\pm}1.6$ & $66.3{\pm}0.6$
& $6.16{\pm}0.07$ & $112.8{\pm}1.8$ & $0.629{\pm}0.006$ \\

\textbf{RCR}
& $\textbf{3.63}{\pm}0.03$ & $\textbf{91.6}{\pm}1.3$ & $\textbf{83.6}{\pm}0.4$
& $\textbf{4.30}{\pm}0.04$ & $\textbf{94.5}{\pm}1.5$ & $\textbf{67.3}{\pm}0.5$
& $\textbf{5.91}{\pm}0.05$ & $\textbf{108.7}{\pm}1.7$ & $\textbf{0.653}{\pm}0.005$ \\
\bottomrule
\end{tabular}
}
\vspace{-2pt}
\end{table*}

\begin{table}[t]
\centering
\caption{
Effect of Iterative Routing Feedback (IRF) across
\textbf{HotPotQA}, \textbf{ALFWorld}, and \textbf{WebShop}.
Each benchmark is evaluated using its official primary metric.
Increasing routing iterations $T$ improves task performance
but increases runtime and token usage.
}
\label{tab:irf_all_vertical}
\setlength{\tabcolsep}{4pt}
\small
\begin{tabular}{lcccc}
\toprule
Benchmark & T & Avg Rt (s)$\downarrow$ & Tok (K)$\downarrow$ & Metric$\uparrow$ \\
\midrule

\multirow{3}{*}{\shortstack[l]{HotPotQA\\(F1, 0--100)}}
& 1 & 91.6  & 3.73 & 83.6 \\
& 2 & 115.8 & 5.92 & 85.0 \\
& 3 & 141.7 & 7.68 & \textbf{85.8} \\
\midrule

\multirow{3}{*}{\shortstack[l]{ALFWorld\\(Success Rate, \%)}} 
& 1 & 94.5  & 4.30 & 67.3 \\
& 2 & 118.5 & 6.28 & 70.9 \\
& 3 & 143.9 & 7.85 & \textbf{71.7} \\
\midrule

\multirow{3}{*}{\shortstack[l]{WebShop\\(Reward, 0--1)}}
& 1 & 108.7 & 5.91 & 0.653 \\
& 2 & 132.7 & 7.54 & 0.680 \\
& 3 & 157.6 & 9.11 & \textbf{0.687} \\
\bottomrule
\end{tabular}
%\vspace{-8pt}
\end{table}

\begin{table}[h]
\centering
\small
\caption{
Controlled comparison of iterative routing feedback and additional computation
on HotPotQA. All $T{=}2$ variants use comparable model-call budgets.
}
\label{tab:irf_disentangle}
\setlength{\tabcolsep}{4pt}
\begin{tabular}{lccc}
\toprule
\textbf{Variant} & \textbf{Rt (s)$\downarrow$} & \textbf{Tok (K)$\downarrow$} & \textbf{F1$\uparrow$} \\
\midrule
Single-pass RCR ($T{=}1$) 
& \textbf{91.6} & \textbf{3.73} & 83.6 \\
Static memory rerouting ($T{=}2$) 
& 109.4 & 5.48 & 84.1 \\
Extra rounds w/o update ($T{=}2$) 
& 112.7 & 5.61 & 83.5 \\
Full-Context same rounds ($T{=}2$) 
& 196.4 & 8.74 & 81.2 \\
IRF w/ memory update ($T{=}2$) 
& 115.8 & 5.92 & \textbf{85.0} \\
\bottomrule
\end{tabular}
\vspace{-7pt}
\end{table}

\begin{table}[t]
\centering
\small
\caption{
Failure-recovery analysis on HotPotQA.
}
\label{tab:failure_recovery}
\begin{tabular}{l|c|c|c|c}
\toprule
{Setting} & 
{Traj. Len} &
{Mem Gro} &
{Tok (K)$\downarrow$} &
{Rt (s)$\downarrow$} \\
\midrule
Normal &
1x &
1x &
3.63 &
91.6 \\
+Retry-1 &
1.25x &
1.25x &
4.12 &
103.5 \\
+Retry-2 &
1.5x &
1.5x &
4.67 &
116.8 \\
\bottomrule
\end{tabular}
%\vspace{-8pt}
\end{table}

\begin{table}[h!]
\centering
\small
\caption{
RCR's routing overhead on HotPotQA.
.
}
\label{tab:routing_overhead}
\begin{tabular}{l|c|c}
\toprule
\textbf{Component} &
\textbf{Latency (s)} &
\textbf{Percentage} \\
\midrule
Scorer &
1.2 &
1.3\% \\
TBA + Memory Selection &
0.8 &
0.9\% \\
Agent LLM Inference &
89.6 &
97.8\% \\
\midrule
Total &
91.6 &
100\% \\
\bottomrule
\end{tabular}
\end{table}

\begin{figure*}[t]
  \centering
  \includegraphics[width=\linewidth]{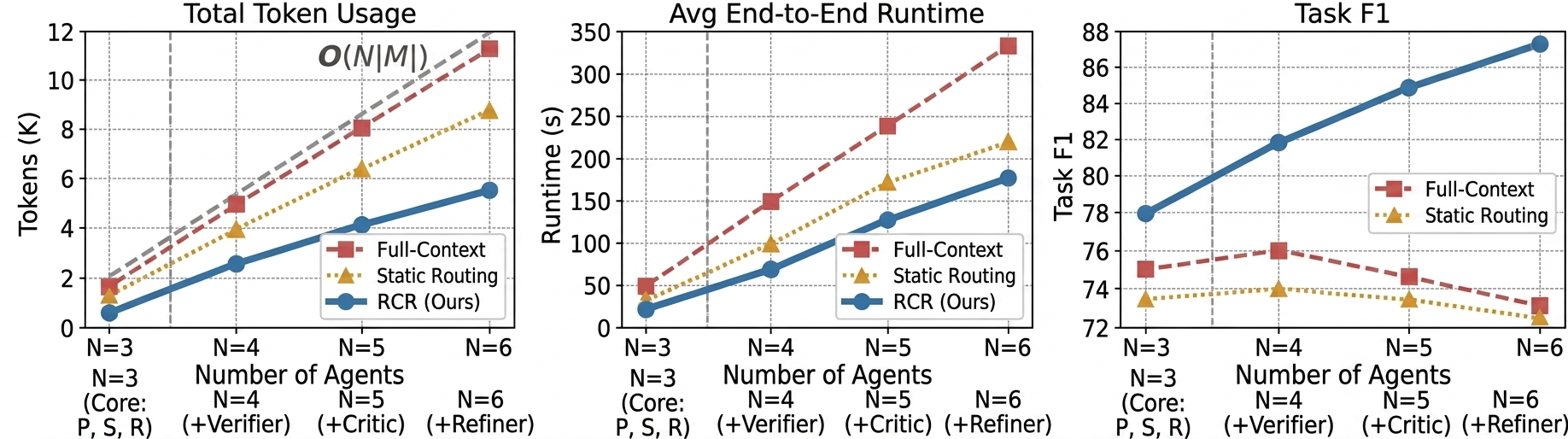}
 \caption{%
    Agent-count sensitivity on HotPotQA as agent count increases
    ($N \in \{3,4,5,6\}$).
    All points are measured under the same evaluation split, prompts,
    stopping rules, and routing implementation as the main experiments;
    the dashed vertical line marks the default $N{=}3$ configuration
    from Table~\ref{tab:combined_results}.
    \textbf{(Left)} Token usage increases with $N$; the gray dashed
    $\mathcal{O}(N|M|)$ line is a reference trend, and RCR grows more slowly  than Full-Context.
    \textbf{(Center)} Runtime follows a similar pattern.
    \textbf{(Right)} RCR maintains stable or improved F1 as auxiliary roles
    are added, while Full-Context shows no consistent gain under uniform
    context exposure.
}
  \label{fig:scaling}
%\vspace{-6pt}
\end{figure*}

\begin{figure}[tb!]
    \centering
    \includegraphics[width=0.95\linewidth]{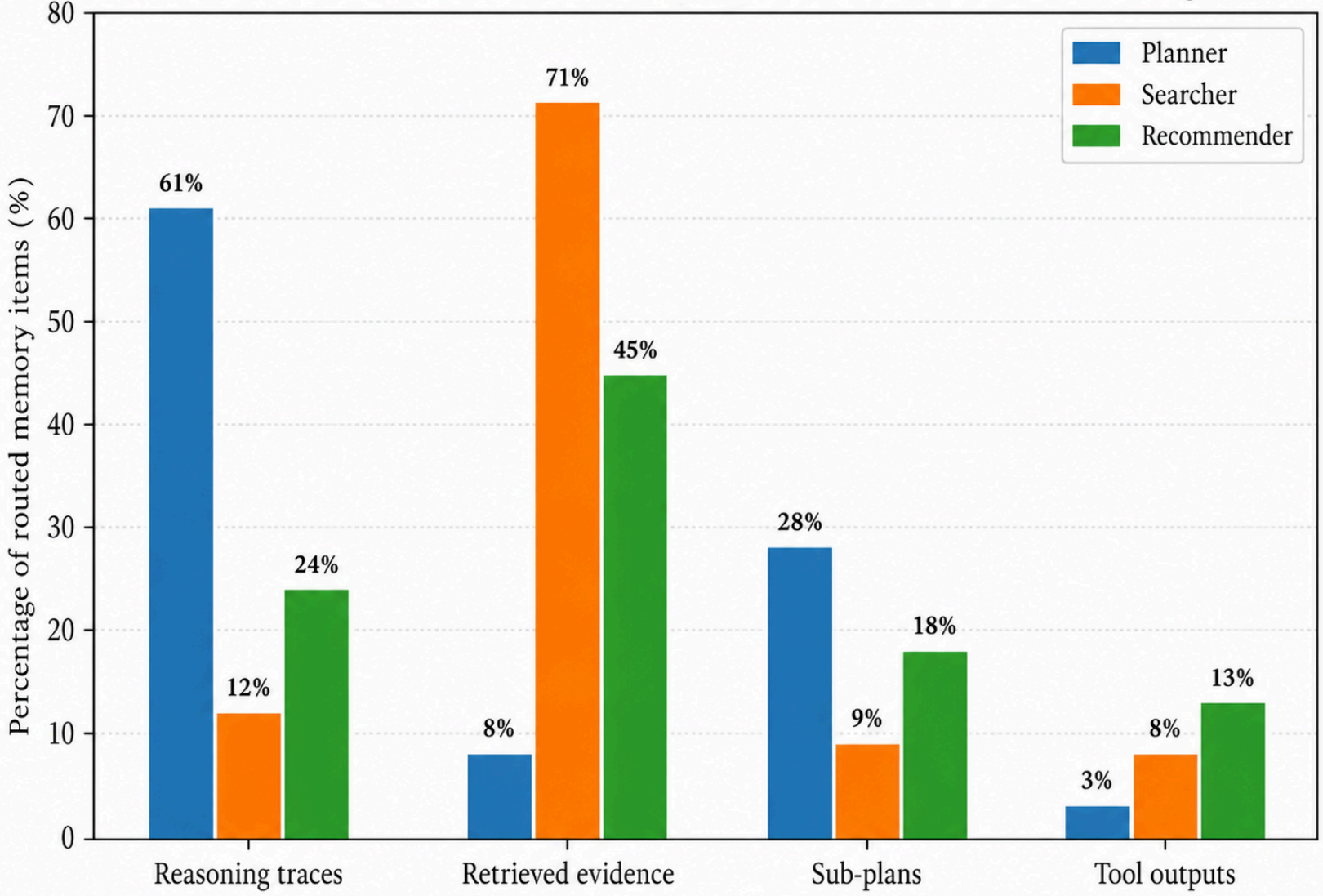}
    \caption{
   Distribution of selected memory types across roles on HotPotQA.
This pattern is consistent with role-conditioned routing: the Planner receives
predominantly reasoning traces and sub-plans, while the Searcher receives
mostly retrieved evidence, suggesting that RCR allocates functionally distinct
memory subsets rather than merely shortening prompts uniformly.
    }
    \label{fig:memory_type_distribution}
%    \vspace{-10pt}
\end{figure}

\begin{table*}[t]
\centering
\caption{
Module-wise ablations under the default single-pass setting ($T{=}1$).
TBA = Token Budget Allocator, IS = Importance Scorer,
and SFR = Semantic Filter and Routing.
Each benchmark is evaluated using its official primary metric.
Removing any module degrades the efficiency--performance trade-off.
}
\label{tab:abl_modules_all}
\small
\begin{tabular}{l|l|c|c|c}
\toprule
Benchmark & Configuration
& Avg Runtime (s)$\downarrow$
& Token (K)$\downarrow$
& Primary Metric$\uparrow$ \\
\midrule

% ---------------------------
\multirow{4}{*}{HotPotQA (F1)}
& Full (TBA + IS + SFR)        & \textbf{91.58} & \textbf{3.63} & \textbf{83.6} \\
& w/o TBA                      & 101.8 & 3.92 & 79.4 \\
& w/o IS (Recency only)        & 96.9  & 3.84 & 78.0 \\
& w/o SFR (FC-in-B)            & 108.7 & 3.77 & 76.5 \\
\midrule

% ---------------------------
\multirow{4}{*}{ALFWorld (Success \%)}
& Full (TBA + IS + SFR)        & \textbf{94.5} & \textbf{4.30} & \textbf{67.3} \\
& w/o TBA                      & 107.3 & 4.56 & 64.4 \\
& w/o IS (Recency only)        & 102.1 & 4.47 & 63.3 \\
& w/o SFR (FC-in-B)            & 111.5 & 4.39 & 62.3 \\
\midrule

% ---------------------------
\multirow{4}{*}{WebShop (Reward 0--1)}
& Full (TBA + IS + SFR)        & \textbf{108.7} & \textbf{5.91} & \textbf{0.653} \\
& w/o TBA                      & 122.4 & 6.28 & 0.627 \\
& w/o IS (Recency only)        & 118.2 & 6.14 & 0.619 \\
& w/o SFR (FC-in-B)            & 126.9 & 6.03 & 0.612 \\
\bottomrule
\end{tabular}
\vspace{-8pt}
\end{table*}

\begin{table*}[t]
\centering
\caption{
Comparison with generic context optimization strategies under identical settings. Runtime/token usage are mean $\pm$ std over five runs; $^{*}$ denotes significant improvement over the strongest baseline ($p<0.05$).
}
\label{tab:context_optimization_comparison}
\small
\begin{tabular}{l|l|c|c|c}
\toprule
Benchmark & Strategy
& Avg Runtime (s)$\downarrow$
& Token (K)$\downarrow$
& Primary Metric$\uparrow$ \\
\midrule

% ---------------------------
\multirow{4}{*}{HotPotQA (F1)}
& MMR Retrieval
& 94.8$\pm$1.2 & 3.72$\pm$0.06 & 81.8$\pm$0.6 \\
& Learned Reranker
& 93.6$\pm$1.1 & 3.70$\pm$0.05 & 82.5$\pm$0.5 \\
& Prompt Compression
& 89.7$\pm$1.0 & 3.10$\pm$0.04 & 81.5$\pm$0.7 \\
& \textbf{RCR}
& \textbf{91.6$\pm$1.3}
& \textbf{3.63$\pm$0.05}
& \textbf{83.6$\pm$0.5}$^{*}$ \\
\midrule

% ---------------------------
\multirow{4}{*}{ALFWorld (Success \%)}
& MMR Retrieval
& 101.2$\pm$1.8 & 4.40$\pm$0.08 & 66.7$\pm$0.8 \\
& Learned Reranker
& 99.8$\pm$1.6 & 4.37$\pm$0.07 & 67.0$\pm$0.7 \\
& Prompt Compression
& 95.1$\pm$1.4 & 3.75$\pm$0.06 & 65.8$\pm$0.9 \\
& \textbf{RCR}
& \textbf{94.5$\pm$1.5}
& \textbf{4.30$\pm$0.07}
& \textbf{67.3$\pm$0.8}$^{*}$ \\
\midrule

% ---------------------------
\multirow{4}{*}{WebShop (Reward 0--1)}
& MMR Retrieval
& 110.9$\pm$2.0 & 6.05$\pm$0.10 & 0.638$\pm$0.008 \\
& Learned Reranker
& 109.5$\pm$1.9 & 6.02$\pm$0.09 & 0.636$\pm$0.007 \\
& Prompt Compression
& 106.2$\pm$1.7 & 5.35$\pm$0.08 & 0.612$\pm$0.009 \\
& \textbf{RCR}
& \textbf{108.7$\pm$1.7}
& \textbf{5.91$\pm$0.09}
& \textbf{0.653$\pm$0.005}$^{*}$ \\

\bottomrule
\end{tabular}
\vspace{-8pt}
\end{table*}

\subsection{Impact of Iterative Feedback Routing (IRF)}
Table~\ref{tab:irf_all_vertical} studies routing iterations
$T \in \{1,2,3\}$. More iterations improve the primary metric but increase
token usage and runtime. While $T{=}3$ gives the best performance, its gain over
$T{=}2$ is modest relative to the added cost. Thus, $T{=}2$ offers a practical
efficiency--performance trade-off, whereas $T{=}1$ remains the default
low-overhead setting.
To disentangle feedback from merely spending more compute,
Table~\ref{tab:irf_disentangle} compares IRF with static-memory rerouting and extra rounds without structured memory updates. IRF yields the strongest F1 among these variants under comparable additional-call settings, suggesting that structured memory updates contribute beyond additional computation alone.

\subsection{Ablation Studies}
\label{sec:Ablation}
We ablate the Token Budget Allocator (TBA), Importance Scorer (IS), and
Semantic Filter \& Routing (SFR) under the same single-pass setting
($T{=}1$) and role configuration.
\textbf{w/o TBA} uses a fixed uniform budget, \textbf{w/o IS} uses recency-only scoring, and \textbf{w/o SFR} selects memory items by insertion order within the budget.
Appendix~\ref{app:budget_diagnostic_all}, Table~\ref{tab:budget_diagnostic_all} shows that learned TBA improves over both uniform and development-tuned static per-role budgets across all three benchmarks.

%To disentangle feedback from merely spending more compute, Table~\ref{tab:irf_disentangle} compares IRF with static-memory rerouting, extra rounds without structured memory updates, and a Full-Context variant using the same round budget. IRF achieves the strongest F1 among comparable additional-round settings, suggesting that structured memory updates contribute beyond additional computation alone. %Role-specific fixed-budget effects are analyzed in Appendix~\ref{subsec:role_budget}.

\subsection{Mechanism analysis.}
RCR's gains are not merely efficiency artifacts. As quantified in
Figure~\ref{fig:memory_type_distribution}, RCR exhibits clear role-wise
separation: the Planner mainly receives reasoning traces ($61\%$) and
sub-plans ($28\%$), while the Searcher mainly receives retrieved evidence
($71\%$). The Recommender receives a more balanced mixture of evidence and
answer-related signals. This functional filtering reduces irrelevant context exposure and helps explain the quality gains in
Table~\ref{tab:combined_results} while accounting for the reduced token cost. 
A concrete routing example is provided in
Appendix~\ref{app:mechanism_analysis}.

\subsection{Agent-Count Sensitivity Analysis}
\label{sec:scalability}
To examine how routing cost changes as the number of agents increases, we
conduct a controlled HotPotQA agent-count analysis. Starting from the default three-agent setup, we incrementally add a \textit{Verifier} ($N{=}4$), \textit{Critic} ($N{=}5$), and \textit{Refiner} ($N{=}6$), while keeping the shared memory pool, prompts, routing implementation, and stopping criteria fixed. The $N{=}3$ baseline is anchored to Table~\ref{tab:combined_results}.

\noindent\textbf{Token and Runtime Scaling.}
Full-Context token usage increases from 5.10\,K at $N{=}3$ to 10.18\,K at
$N{=}6$, reflecting repeated full-memory exposure as agents are added.
RCR grows more slowly (3.63\,K $\to$ 5.98\,K), as the Token Budget Allocator assigns smaller budgets to narrowly scoped agents, consistent with the $\mathcal{O}(\sum_i B_i)$ routing-cost bound. Runtime shows the same trend: 301.8\,s for Full-Context versus 137.6\,s for RCR at $N{=}6$
(Figure~\ref{fig:scaling}, Appendix Table~\ref{tab:scaling}).

\noindent\textbf{Task Performance.}
RCR maintains stable or improved F1 as auxiliary roles are added (83.6 $\to$ 85.7), while Full-Context peaks at $N{=}4$ and degrades, preserving RCR's advantage at scale.

\subsection{Failure Recovery Scenario}
In Table~\ref{tab:failure_recovery}, we constructed a failure-recovery variant on HotPotQA: after injecting a failure at a random intermediate step, we trigger 1 or 2 additional retry cycles (re-planning and re-execution) and measure the resulting cost. As shown below, token usage and runtime increase roughly proportionally with the number of retries, while RCR continues to route context under its existing budget and relevance mechanism without requiring architectural changes.

\subsection{Computational Overhead.}
In Table~\ref{tab:routing_overhead}, we isolate the routing layer's own overhead on HotPotQA:
the Importance Scorer accounts for 1.2s (1.3\%), and the Token Budget Allocator
plus size-aware selection account for 0.8s (0.9\%), together contributing 2.2\%
of RCR's total 91.6s runtime. The remaining 97.8\% (89.6s) is agent LLM inference.
This confirms that the routing components are lightweight relative to LLM calls,
and that RCR's net runtime reduction over Full-Context (91.6s vs. 150.65s)
comes primarily from shorter, more relevant LLM contexts rather than from routing
computation itself. We will add this breakdown to the revised paper.

\begin{table}[tb!]
\centering
\small
\caption{
Role granularity sensitivity on HotPotQA.
}
\label{tab:role_granularity}
\begin{tabular}{l|l|c|c|c|c}
\toprule
\textbf{Level} & \textbf{Role Decom} &
\textbf{Rol} & \textbf{Rt (s)$\downarrow$} &
\textbf{Tok $\downarrow$} & \textbf{F1$\uparrow$} \\
\midrule

Coarse &
\shortstack[l]{Reasoner; \\ Executor}  &
2 & 93.1 & 3.75 & 82.4 \\

\midrule

Default &
\shortstack[l]{Planner; \\Searcher;\\  Rec} &
3 & 91.6 & 3.63 & 83.6 \\
\midrule

Fine &
\shortstack[l]{Planner; \\Searcher;\\
 Checker;\\  Composer;\\
Verifier} &
5 & 96.2 & 4.02 & 83.8 \\

\bottomrule
\end{tabular}
\vspace{-10pt}
\end{table}

\subsection{Role Configuration Sensitivity.}
Role configurations vary across benchmarks: HotPotQA uses three roles, while ALFWorld and WebShop use four. RCR remains robust across role taxonomies, with role decomposition generated by an LLM-assisted agent (Appendix~F.1) and shared across methods. The importance scorer uses free-form role descriptions. Table~\ref{tab:role_granularity} shows only 1.4-point F1 variation (82.4--83.8) from 2-role coarse to 5-role fine decompositions, with gradual cost increase. Together with the agent-count study (Section~\ref{sec:scalability}), these results demonstrate robustness to role count and decomposition.

%\clearpage
\section{Conclusion}
We presented RCR, an execution-time context allocation layer that formulates memory routing as a role-conditioned, token-budgeted selection problem.
By combining budget allocation, relevance scoring, size-aware routing, and iterative memory updates, RCR reduces redundant context exposure while maintaining task performance across multi-hop QA, ALFWorld, and WebShop, outperforming full-context, 
static-routing, retrieval-based, debate-based, and controlled routing baselines.

\clearpage
\section*{Limitations}
RCR does not optimize the agent topology or design the roles themselves.
Instead, it focuses on deciding which shared-memory items should be exposed to
each existing agent under resource constraints. Accordingly, RCR assumes a
given multi-agent topology and uses available agent-level metadata, such as
role identifiers or task-stage tags, for context allocation.

Our experiments use fixed multi-agent topologies, with agent-count sensitivity
studied for $N\in\{3,4,5,6\}$. These results suggest that RCR can preserve its
efficiency advantage within the evaluated range of agent counts, but they do
not fully characterize behavior in much larger systems with more roles, longer
interaction horizons, or more complex coordination patterns.

Our evaluation covers multi-hop reasoning on HotPotQA~\cite{yang2018hotpotqa},
embodied decision-making on ALFWorld~\cite{ALFWorld20}, and web navigation on
WebShop~\cite{yao2022webshop}, with additional multi-hop QA results on
MuSiQue~\cite{trivedi2022musique} and
2WikiMultiHop~\cite{xanh2020_2wikimultihop} reported in
Appendix~\ref{app:multi-hop_QA}. This provides preliminary evidence of
applicability beyond a single task family, but future work should evaluate RCR
in additional domains such as long-horizon tool use, collaborative coding, and
larger-scale multi-agent workflows.

%Token reduction is our primary reproducible efficiency metric, while runtime should be interpreted as an environment-dependent measure of end-to-end latency.

Finally, our main experiments use GPT-5 mini~\cite{openai_gpt5mini_2025} as
the primary agent backbone. Appendix~\ref{app:backbone_transfer} reports
backbone-transfer experiments on Gemini 2.5 Pro~\cite{geminipro25_2025} and
Qwen3-Max~\cite{qwen3max_2025}. These results reduce, but do not eliminate,
the concern that RCR is tied to a specific LLM backbone. Future work could
further test broader open-weight, closed-source, multilingual, and 
domain-specialized backbones.

Furthermore, we will investigate advanced techniques such as model pruning~\cite{ji2025computation,li2025mutual,liu2025toward}, which enable neural networks to maintain computational efficiency as the number of parameters increases, particularly in resource-limited settings~\cite{liu2025rora,niu2025mobile,tan2025zeroqat,zhang2025towards,tan2025harmony,10.1145/3747842,tan2025perturbation,yang2025fairsmoe,yuan2021work,yuan2022mobile,liu2023scalable,liu2026should}. we will employ diffusion-based generative models~\cite{meng2024instructgie} to synthesize compact, informative samples~\cite{liu2023interpretable,liu2021explainable,liu2022efficient}, thereby facilitating scalable experimentation in healthcare-related tasks~\cite{liu2024brain,wang2025fairgnn,chinta2023optimization,wang2025graph,wang2025amcr,chinta2025ai}. These synthetic samples will be integrated with large language models (LLMs) to support downstream applications such as 3D reconstruction~\cite{lei2023mac,li2025local,dong2024physical,dong2024df}, enhancing overall model utility in multi-modal scenarios~\cite{xu2026lever,dong2026imageslookrightretrieve,liu2026dualpathwaycircuitsobjecthallucination,zhang-etal-2026-modality,kang2026drilltrainingworldmodels,wang-etal-2026-memguard,zhang-etal-2026-memmark,liu2026conmemcontributionawarememorylonghorizon,wang2026parprogressalignedrecurrencelength,dong2026conditionalrankallocationtaxonomyaware}.

\bibliography{custom}

%\clearpage
\appendix

%\section{Example Appendix}
%\label{sec:appendix}

\iffalse
\section{Example Appendix}
\label{sec:appendix}

In implementation, we use PPO-style clipped updates for stable optimization:
\begin{equation}
\mathcal{L}_{\mathrm{PPO}}
=
-\mathbb{E}
\left[
\min
\left(
\rho_t A_t,
\mathrm{clip}(\rho_t,1-\epsilon,1+\epsilon)A_t
\right)
\right],
\end{equation}
where $\rho_t=
\pi_\theta(B_i \mid r_i,s_t,u_i)/
\pi_{\theta_{\mathrm{old}}}(B_i \mid r_i,s_t,u_i)$ and
$A_t=R(\tau)-b$.
\fi

\section{Algorithm}
Algorithm~\ref{alg:rcr_context_routing} summarizes the routing workflow. 
\begin{algorithm}[tb]
   \caption{RCR Context Routing}
   \label{alg:rcr_context_routing}
\begin{algorithmic}
   \STATE {\bfseries Input:} Memory $M$, Role $R_i$, Stage $S_t$, Budget $B_i$
   \STATE {\bfseries Output:} Routed context $C_t^i$
   \STATE $C_{\mathrm{set}} \gets \emptyset$, \; $rem \gets B_i$
   \STATE $M' \gets \{m \in M \mid \mathrm{Len}(m) \le B_i\}$
   \IF{$M'=\emptyset$}
       \STATE \textbf{return} $\emptyset$
   \ENDIF
   \FOR{each $m \in M'$}
       \STATE $\alpha(m) \gets \mathrm{Score}(m, R_i, S_t)$
   \ENDFOR
   \STATE Sort $M'$ by descending $\alpha(m)/\mathrm{Len}(m)$
   \FOR{each $m \in M'$}
       \IF{$\mathrm{Len}(m) \le rem$}
           \STATE $C_{\mathrm{set}} \gets C_{\mathrm{set}} \cup \{m\}$; \; $rem \gets rem - \mathrm{Len}(m)$
       \ENDIF
   \ENDFOR
   \STATE $C_{\mathrm{single}} \gets \{\arg\max_{m \in M'} \alpha(m)\}$
   \STATE $C_t^i \gets \arg\max_{C \in \{C_{\mathrm{set}}, C_{\mathrm{single}}\}} \sum_{m \in C}\alpha(m)$
   \STATE \textbf{return} $C_t^i$
\end{algorithmic}
\end{algorithm}

\section{Additional Implementation Details}
\label{app:Implementation}

\subsection{Benchmarks and Metrics.}
We evaluate our system on three multi-hop QA benchmarks:
\textit{HotPotQA}~\cite{yang2018hotpotqa},
\textit{MuSiQue}~\cite{trivedi2022musique}, and
\textit{2WikiMultihopQA}~\cite{xanh2020_2wikimultihop};
one embodied reasoning benchmark,
\textit{ALFWorld}~\cite{ALFWorld20}; and one web navigation benchmark,
\textit{WebShop}~\cite{yao2022webshop}.
All reported results are averaged over 5 independent runs.

\subsection{ Experimental Setup}
To ensure comparability, all controlled multi-agent methods use the same agent roles, role prompts, task prompts, tool interfaces, maximum interaction rounds, and stopping conditions within each benchmark. The only varied component is the context routing policy. Full-Context routes the entire shared memory to every agent, Static Routing uses fixed role-specific context templates, and Multi-Agent Debate uses the same roles and round limit but shares identical context among debating agents. Retrieval-based baselines use the same memory pool and token budget, but retrieve a single shared context rather than role-conditioned contexts.

\noindent\textbf{Dataset Statistics.}
We follow the official standard settings released with each benchmark.
The number of trajectories used for router training is already reported in Appendix~D, Table~8 ("Weak-label statistics for training the importance scorer"): 1,000 for HotPotQA, 800 for ALFWorld, and 1,000 for WebShop. We consolidate this information alongside dataset split sizes and test instance counts in a single summary table, Table~\ref{tab:dataset_statistics}, for easier reference, as suggested by the reviewer.

\begin{table*}
\centering
\small
\caption{
Dataset statistics and routing data configuration.
}
\label{tab:dataset_statistics}
\begin{tabular}{l|c|c|c|c|c}
\toprule
\textbf{Benchmark} &
\textbf{Train/Val/Test} &
\textbf{Router Traj.} &
\textbf{Test Inst.} &
\textbf{Mem Items} &
\textbf{Mem Tokens} \\
\midrule

HotPotQA &
90,447 / 7,405 / 7,405 (fullwiki) &
1,000 &
7,405 &
18.6 &
3.2K \\

ALFWorld &
3,553 / 140 / 134 &
800 &
134 &
24.3 &
4.1K \\

WebShop &
10,587 / 1,000 / 500 &
1,000 &
500 &
31.7 &
5.6K \\

\bottomrule
\end{tabular}
\end{table*}

\subsection{Training.}
\label{app:training}
\noindent\textbf{Budget Allocator Training.}
The token budget allocator is trained with a PPO-style policy-gradient update over $\mathcal{B}=\{512,1024,2048,4096\}$.
Its reward combines the benchmark-specific primary metric $Q(\tau)$ with token and latency penalties. We use a moving-average baseline and entropy
regularization; budgets are sampled during training and chosen by maximum
policy probability at inference. Full optimization details are provided in Appendix~\ref{app:rl_allocator}.

\noindent\textbf{Importance Scorer Training.}
The importance scorer is trained with AdamW
($5\times10^{-4}$ learning rate, $10^{-4}$ weight decay, batch size 256,
cosine decay with 10\% warm-up). Because the LLM is black-box and routing is discrete, we use weak supervision from logged trajectories. Positives are explicitly supporting memory items;
negatives are contradicted items or same-role--stage hard negatives. Ambiguous unused items are discarded, and labels use only information available up to the current round, without gold evidence or future outcomes. The scorer is optimized with a binary relevance loss conditioned on role and task stage. Details are provided in Appendix~\ref{app:scorer}.

\subsection{Baseline Normalization.}
%We distinguish between controlled multi-agent routing baselines and retrieval-oriented reference baselines.
For controlled multi-agent baselines, including Full-Context, Static Routing, Multi-Agent Debate, and RCR, we use the same agent topology, role prompts, task prompts, shared memory pool, tool interface, maximum number of model calls, and stopping condition within each benchmark.
The only varied component is the context routing policy.
Full-Context sends the entire shared memory to every agent, Static Routing uses fixed role-specific context allocation, Multi-Agent Debate uses the same agent set but shares identical context during critique, and RCR performs role-conditioned budgeted routing.

%Self-RAG and End-to-End Beam Retrieval are included as retrieval-oriented reference baselines rather than strict multi-agent routing baselines.
For these methods, we use the same underlying memory/evidence pool and match the total prompt budget and number of model calls whenever applicable.
They retrieve a single shared context for generation. 
%whereas RCR allocates different context subsets to different agents.

\subsection{Split Protocol and Leakage Prevention.}
For each benchmark, we construct disjoint train, validation, and test splits at the task-instance level. The allocator and importance scorer are trained only on trajectories collected from the training split. Hyperparameters, including the default budget and routing thresholds, are selected on the validation split. All numbers in Tables are reported on the held-out test split. No test trajectories, final answers, gold evidence labels, or post-hoc task outcomes are used for training the scorer or allocator. During evaluation, the shared memory is initialized only with information available to the agents before or during execution, and memory updates are generated online from agent outputs, preventing leakage from future rounds or gold annotations.

\subsection{Reward Coefficient Tuning}
\label{app:lambda_sensitivity}

The reward coefficients are tuned on the development 
split via grid search over $\lambda_{\mathrm{Tok}}, \lambda_{\mathrm{Lat}} 
\in \{0.01, 0.05, 0.1\}$, yielding $\lambda_{\mathrm{Tok}} = 0.05$ and 
$\lambda_{\mathrm{Lat}} = 0.01$ for HotPotQA and ALFWorld, and $\lambda_{\mathrm{Tok}} = 0.05$ and $\lambda_{\mathrm{Lat}} = 0.05$ for WebShop.
To make all reward terms comparable, $Q(\tau)$ is normalized to $[0,1]$ across benchmarks, and $\mathrm{Tok}(\tau)$ and $\mathrm{Lat}(\tau)$ are standardized by development-set mean and standard deviation before applying the coefficients.

Table~\ref{tab:lambda_sensitivity_hotpotqa} reports a small sensitivity study
on the HotPotQA development split. Larger token or latency penalties reduce cost
but can slightly lower F1. We select $(\lambda_{\mathrm{Tok}},\lambda_{\mathrm{Lat}})
=(0.05,0.01)$ because it gives the best development-set efficiency--performance
trade-off, and keep this setting fixed for test evaluation.

\begin{table}[t]
\centering
\small
\caption{
Sensitivity to reward coefficients on the HotPotQA development split.
The selected setting is marked in bold. Token and runtime are measured from
actual execution.
}
\label{tab:lambda_sensitivity_hotpotqa}
\begin{tabular}{ccccc}
\toprule
$\lambda_{\mathrm{Tok}}$ & $\lambda_{\mathrm{Lat}}$ 
& F1$\uparrow$ & Token (K)$\downarrow$ & Runtime (s)$\downarrow$ \\
\midrule
0.01 & 0.01 & 83.4 & 3.91 & 96.8 \\
0.05 & 0.01 & \textbf{83.6} & \textbf{3.63} & \textbf{91.6} \\
0.10 & 0.01 & 83.1 & 3.41 & 88.9 \\
0.05 & 0.05 & 83.2 & 3.58 & 89.7 \\
0.05 & 0.10 & 82.7 & 3.49 & 87.5 \\
\bottomrule
\end{tabular}
\end{table}

\subsection{Role Configuration.}
\label{app:configuration}
Each benchmark adopts a task-specific cooperative-agent configuration following the instruction-based agent abstraction used in the OpenAI Agents SDK~\cite{openai_agents_sdk}. 
For {multi-hop QA} tasks (HotPotQA, MuSiQue, 2WikiMultihop), 
we employ three functional roles: Planner, Searcher, and Recommender,
responsible for query decomposition, evidence retrieval, and answer synthesis, respectively.

For {embodied reasoning} (ALFWorld), we extend to a four-role setup: 
Planner, Navigator, Interactor, and Verifier, 
corresponding to high-level planning, environment navigation, object manipulation, 
and goal verification phases during textual interaction with the environment.

For the \textit{multi-turn recommendation} benchmark (WebShop), 
we adopt a conversational four-agent configuration consisting of 
a \textbf{Buyer} (expresses preferences and constraints such as price, style, or brand), 
a \textbf{Retriever} (searches the catalog for matching products), 
a \textbf{Rater} (evaluates candidate items based on metadata and user alignment), 
and a \textbf{Recommender} (produces the final ranked list with concise justifications).

All roles share a global memory store $M_t$ and communicate through
role-conditioned context routing in the RCR framework.  RCR uses these fixed role and stage annotations only as routing metadata and does not optimize the agent topology. In RCR, the trained Token Budget Allocator selects each agent's budget $B_i$ from
$\{512,1024,2048,4096\}$, while $B_i{=}2048$ serves as the reference budget for controlled baseline comparisons. Detailed prompt examples are provided in Appendix~\ref{app:role_config}.

\subsection{Reproducibility Checklist}
\label{app:repro_checklist}
See Table~\ref{tab:repro_checklist}.
\begin{table*}[t]
\centering
\small
\caption{Reproducibility checklist for controlled experiments.}
\label{tab:repro_checklist}
\begin{tabular}{l|p{0.58\linewidth}}
\toprule
\textbf{Component} & \textbf{Specification} \\
\midrule
Agent interface 
& Same role descriptions, task prompts, tool interfaces, round limits, and stopping rules are used across controlled multi-agent methods. \\

Role metadata 
& Roles are treated as metadata provided by the underlying agent framework or decomposition module. RCR does not require a fixed role taxonomy. \\

Stage metadata 
& Stage descriptors are used as optional routing metadata. Baselines that do not use stage-conditioned routing ignore this metadata. \\

Memory pool 
& All controlled methods use the same shared-memory pool within each benchmark. \\

Routing policy 
& The only varied component is how memory items are selected and routed to agents. \\

%Retrieval baselines 
%& Self-RAG and End-to-End Retrieval are baselines using the same memory/evidence pool and comparable total prompt budgets.\\ %but are not strict multi-agent routing baselines. \\

Evaluation 
& QA benchmarks use Precision/Recall/F1. ALFWorld and WebShop use their benchmark-specific success or task-score metrics. Runtime and token usage are reported across all benchmarks. \\
\bottomrule
\end{tabular}
\end{table*}

\section{Reward-Optimized Budget Allocation}
\label{app:rl_allocator}

\paragraph{Policy Inputs and Actions.}
At interaction round $t$, for agent $A_i$, the budget allocator observes
\begin{equation}
o_t^i = [r_i; s_t; u_i],
\end{equation}
where $r_i$ is the role embedding, $s_t$ is the stage embedding, and $u_i$
summarizes recent token consumption, context utilization, latency, and task
feedback. The allocator outputs a distribution over a discrete budget set
\begin{equation}
\mathcal{B}=\{512,1024,2048,4096\}.
\end{equation}
A budget is sampled as
\begin{equation}
B_i \sim \pi_\theta(\cdot \mid o_t^i).
\end{equation}
Given $B_i$, RCR orders memory items by normalized utility
$\alpha(m;R_i,S_t)/\ell(m)$ and applies the size-aware online selector.

\paragraph{Reward Function.}
%After completing a multi-agent trajectory $\tau$, the allocator receives an episode-level reward in Eq.~\ref{eq:reward_function}.
For each benchmark, $Q(\tau)$ is instantiated as the official primary metric. We normalize $\mathrm{Tok}(\tau)$ and $\mathrm{Lat}(\tau)$ using development-set statistics before reward aggregation, so the task-quality and cost terms remain on comparable scales. The coefficients $\lambda_{\mathrm{Tok}}$ and $\lambda_{\mathrm{Lat}}$ are selected on the development split and fixed for test evaluation.

\paragraph{Trajectory Logs.}
Allocator training uses logged trajectories collected under the same agent
roles, prompts, tool interfaces, memory pool, and stopping rules as the main
experiments. To avoid self-labeling from RCR, we initialize training with logs
generated by non-RCR policies, including Full-Context and Static Routing. Each
trajectory records the shared memory state, routed contexts, agent outputs,
final task metric, token usage, and latency. These logs provide the initial
state-action-outcome tuples for policy optimization; the allocator is then
refined using the reward in Eq.~(\ref{eq:reward_function}).

%\paragraph{Policy-Gradient Objective.}
\paragraph{Training Objective.}
%We optimize the allocator using a policy-gradient objective:
We train the allocator with the following policy-gradient objective:
\begin{equation}
\mathcal{L}_{\mathrm{PG}}
=
-
\sum_{t,i}
\log \pi_\theta(B_i \mid o_t^i)
\left(R(\tau)-b\right),
\end{equation}
where $b$ is a moving-average baseline used for variance reduction. The baseline
is updated as
\begin{equation}
b \leftarrow \beta b + (1-\beta) R(\tau),
\end{equation}
where $\beta$ is a smoothing coefficient.

\paragraph{PPO-Style Stabilization.}
For stable optimization, we use a PPO-style clipped update. Let
$\theta_{\mathrm{old}}$ denote the policy parameters before the update. The
importance ratio is defined as
\begin{equation}
\rho_{t,i}
=
\frac{
\pi_{\theta}(B_i \mid o_t^i)
}{
\pi_{\theta_{\mathrm{old}}}(B_i \mid o_t^i)
}.
\label{eq:ppo_ratio}
\end{equation}
The trajectory-level advantage is
\begin{equation}
A(\tau)=R(\tau)-b .
\label{eq:ppo_advantage}
\end{equation}
The clipped PPO loss is
\begin{equation}
\begin{aligned}
\mathcal{L}_{\mathrm{PPO}}
=
-\sum_{t,i}
\min \Big(
& \rho_{t,i} A(\tau), \\
& \operatorname{clip}(\rho_{t,i},1-\epsilon,1+\epsilon) A(\tau)
\Big).
\end{aligned}
\label{eq:ppo_loss}
\end{equation}
We additionally use entropy regularization to encourage exploration:
\begin{equation}
\mathcal{L}_{\mathrm{alloc}}
=
\mathcal{L}_{\mathrm{PPO}}
-
\eta
\sum_{t,i}
\mathcal{H}\!\left(
\pi_{\theta}(\cdot \mid o_t^i)
\right),
\label{eq:alloc_loss}
\end{equation}
where $\eta$ is the entropy coefficient.

\paragraph{Inference.}
During inference, we use deterministic budget selection:
\begin{equation}
B_i
=
\arg\max_{B\in\mathcal{B}}
\pi_{\theta}(B \mid o_t^i).
\label{eq:budget_inference}
\end{equation}
The selected budget is then passed to the routing module. Thus, the allocator
controls how much context each agent receives, while the importance scorer
determines which memory items are selected under that budget.

\section{Importance Scorer Architecture}
\label{app:scorer}

\paragraph{Inputs.}
At interaction round $t$, for a memory item $m$ and agent $A_i$ with role
$R_i$ and stage $S_t$, we construct three textual inputs:
the memory text $x_m$, the role description $x_{R_i}$, and the stage description $x_{S_t}$. We also compute the memory age $\Delta t$, defined as the number of interaction rounds since $m$ was generated.

\paragraph{Shared Text Encoder.}
We use a compact Transformer encoder initialized from MiniLM-L6~\cite{wang2020minilm}
to encode memory, role, and stage inputs:
\begin{equation}
\begin{aligned}
e_m &= \mathrm{Enc}(x_m),\\
e_R &= \mathrm{Enc}(x_{R_i}),\\
e_S &= \mathrm{Enc}(x_{S_t}).
\end{aligned}
\end{equation}
The encoder is shared across the three inputs to reduce parameters and improve
training stability. To distinguish input types, we prepend lightweight type
prefixes such as \texttt{[MEM]}, \texttt{[ROLE]}, and \texttt{[STAGE]} before encoding. Unless otherwise stated, the encoder uses 3 Transformer layers, 8 attention heads, and hidden size 512.

\paragraph{Temporal Encoding.}
To model recency, we encode the memory age $\Delta t$ as a temporal feature:
\begin{equation}
e_T = \mathrm{TimeEmb}(\Delta t).
\end{equation}
In our implementation, $\mathrm{TimeEmb}(\cdot)$ is a learnable embedding table
over discretized interaction ages, with a fallback bucket for memories older
than the maximum tracked age. We use a temporal embedding dimension of 32 and
project it to the same dimension as the text embeddings before fusion.

\paragraph{Role- and Stage-Conditioned Fusion.}
The scorer combines memory, role, stage, and temporal features through gated
conditioning:
\begin{equation}
\tilde{e}_m
=
e_m
\odot
\sigma(W_R e_R)
\odot
\sigma(W_S e_S)
\odot
\sigma(W_T e_T),
\end{equation}
where $\odot$ denotes elementwise multiplication and $\sigma(\cdot)$ is the
sigmoid function. This gating mechanism conditions the memory representation
on both the agent role and the current task stage.

We then form the scoring representation:
\begin{equation}
z_m = [\tilde{e}_m; e_R; e_S; e_T],
\end{equation}
where $[\cdot;\cdot]$ denotes concatenation.

\paragraph{Scoring Head.}
The final relevance score is computed by a two-layer MLP:
\begin{equation}
\alpha(m;R_i,S_t)
=
\sigma
\left(
w_2^\top
\mathrm{ReLU}
(W_1 z_m + b_1)
+
b_2
\right),
\end{equation}
where $\alpha(m;R_i,S_t)\in(0,1)$ represents the relevance of memory item $m$
for role $R_i$ and stage $S_t$. During routing, memory items are ranked by
normalized utility:
\begin{equation}
u(m;R_i,S_t)
=
\frac{\alpha(m;R_i,S_t)}{\ell(m)},
\end{equation}
where $\ell(m)$ is the token length of $m$.

\paragraph{Importance Scorer Label Construction.}
The scorer is trained as a role- and stage-conditioned relevance ranker over
memory items using logged training trajectories. Positive examples are memory
items explicitly linked to successful outputs, evidence chains, tool outputs, or
action traces. High-confidence negatives are restricted to contradicted items
and hard negatives sampled from the same role--stage memory pool. We do not
treat all unused items as negatives, since unused memory items may still be
potentially relevant but omitted due to budget limits or because an equivalent
item was selected. Instead, unused items are triaged into soft negatives or
unlabeled masked items according to confidence.

Labels are constructed only from training-split trajectory records. Gold
answers, gold evidence annotations, test trajectories, and future-round
information unavailable at the current routing step are not used for scorer
training.

We use weighted BCE to account for weak-label confidence:
\begin{equation}
\mathcal{L}_{\mathrm{score}}
=
-\sum_{i,t,m} w_m^{i,t}\, b_m^{i,t},
\end{equation}
where
\begin{equation}
\begin{aligned}
b_m^{i,t}
=&\; y_m^{i,t}\log \alpha_m^{i,t} \\
&+ (1-y_m^{i,t})\log(1-\alpha_m^{i,t}).
\end{aligned}
\end{equation}

where $y_m^{i,t}\in\{0,1\}$ is the weak relevance label,
$\alpha_m^{i,t}=\alpha(m;R_i,S_t)$, and $w_m^{i,t}$ is the label-confidence
weight:
\begin{equation}
(y_m^{i,t},w_m^{i,t}) =
\begin{cases}
(1,\;1.0), & \text{high-confidence positive},\\
(0,\;1.0), & \text{high-confidence negative},\\
(0,\;\eta), & \text{soft negative},\\
(\cdot,\;0), & \text{unlabeled ambiguous item}.
\end{cases}
\end{equation}

%We select $\eta=0.3$ on the development split; sensitivity to $\eta$ is reported in Appendix~\ref{app:eta_sensitivity}.
We select $\eta = 0.3$ on the development split.
Soft negatives are unused items whose role--stage mismatch or low similarity
to successful evidence suggests likely irrelevance, but which are not
contradicted by trajectory evidence with full confidence.
Fully ambiguous items are masked out from the scorer loss rather than treated
as negatives.

\paragraph{Weak-Label Triage.}
Table~\ref{tab:weak_label_stats} summarizes the resulting weak-label
statistics. Items explicitly supporting successful evidence chains, tool
outputs, or action traces are labeled as positives. Contradicted items and
same-role--stage hard negatives are labeled as hard negatives. Soft negatives
are identified by role--stage tag mismatch or cosine similarity below $0.3$ to
any positive item in the same trajectory round; they are retained in training
with reduced loss weight~$\eta$. Items whose relevance cannot be reliably
inferred from the logged trajectory are left unlabeled and assigned zero loss
weight. Thus, the ``Unlabeled'' column in Table~\ref{tab:weak_label_stats}
does not indicate that these items are treated as irrelevant; rather, they are
excluded from the hard-label loss to avoid introducing false-negative noise.

\begin{table}[t]
\centering
\small
\caption{
Weak-label statistics for training the importance scorer.
We use a three-way triage: high-confidence positives, high-confidence
negatives, soft negatives (reduced loss weight $\eta$), and unlabeled
ambiguous items (zero loss weight).
}
\label{tab:weak_label_stats}
\resizebox{\linewidth}{!}{
\begin{tabular}{lcccccc}
\toprule
\textbf{Benchmark} & \textbf{Traj.} & \textbf{Items}
& \textbf{Pos.} & \textbf{Hard Neg.} & \textbf{Soft Neg.} & \textbf{Unlabeled} \\
\midrule
HotPotQA & 1,000 & 48.2K & 18.6\% & 31.4\% & 17.3\% & 32.7\% \\
ALFWorld & 800   & 39.5K & 16.9\% & 29.8\% & 16.6\% & 36.7\% \\
WebShop  & 1,000 & 56.7K & 15.4\% & 33.1\% & 17.0\% & 34.5\% \\
\bottomrule
\end{tabular}}
\end{table}

Table~\ref{tab:scorer_label_examples} gives concrete examples of weak-label
construction across benchmarks. These examples illustrate that negative labels
are not arbitrary unused items: high-confidence negatives are contradicted
records or hard negatives from the same role--stage pool, while lower-confidence
unused records are treated as soft negatives or left unlabeled.

\begin{table*}[t]
\centering
\small
\caption{
Examples of weak-label construction for the importance scorer.
Labels are derived from logged training trajectories without gold evidence
annotations, test trajectories, or future-round leakage.
}
\label{tab:scorer_label_examples}
\resizebox{\linewidth}{!}{
\begin{tabular}{p{0.12\linewidth}|p{0.38\linewidth}|p{0.13\linewidth}|p{0.29\linewidth}}
\toprule
\textbf{Benchmark} & \textbf{Memory item} & \textbf{Label type} & \textbf{Reason} \\
\midrule

HotPotQA
& Retrieved passage cited by the Searcher and later used by the Synthesizer in a successful answer trajectory.
& Positive
& Linked to the logged evidence chain used in the successful trajectory. \\

HotPotQA
& Entity-overlapping passage retrieved for the same question but contradicted by the later accepted evidence.
& Hard negative
& Same-role--stage candidate whose content conflicts with the logged successful evidence. \\

HotPotQA
& Unused passage with partial entity overlap but no clear support or contradiction in the trajectory.
& Unlabeled
& Relevance cannot be reliably inferred; masked out from the scorer loss to avoid false-negative noise. \\

\midrule

ALFWorld
& Action/state trace: \texttt{open fridge; put apple in fridge}.
& Positive
& Appears in a successful logged action trajectory satisfying the goal condition. \\

ALFWorld
& Action/state trace: \texttt{put apple on counter} when the target goal requires storing the apple in the fridge.
& Hard negative
& Contradicts the target state observed in the successful trajectory. \\

ALFWorld
& Unused navigation observation from a different room with weak lexical overlap to the goal.
& Soft negative
& Likely irrelevant due to role--stage mismatch, but not contradicted with full confidence; included with reduced loss weight. \\

\midrule

WebShop
& Candidate product metadata matching logged user constraints, e.g., white blackout shade, correct size, and cordless option.
& Positive
& Used by the Recommender in the final selected product rationale. \\

WebShop
& Candidate product metadata violating a hard user constraint, e.g., wrong size or non-blackout material.
& Hard negative
& Same-stage candidate inconsistent with logged user constraints. \\

WebShop
& Unselected candidate with similar category but incomplete evidence about size or material.
& Soft negative
& Likely less useful than selected evidence but not confidently irrelevant; included with reduced loss weight. \\

\bottomrule
\end{tabular}}
\end{table*}

\section{Memory Update Agent Example}
\label{appendix:update_example}

Rather than manually writing task-specific parsers, we implement memory
construction as a dedicated structured-output agent following the standard
agent abstraction with instructions and output types~\cite{openai_agents_sdk}.
The Memory Update Agent converts raw agent outputs into lightweight records:
\[
m=\{\text{role\_tag},\text{stage\_tag},\text{timestamp},\text{content}\}.
\]
RCR itself is agnostic to how these records are produced; they may come from a
dedicated structured-output agent, a rule-based parser, or an external memory
system.

\paragraph{Update Example.}
Given a raw Searcher output:
\begin{verbatim}
The evidence states that Dataset Y contains 
12K labeled samples.
\end{verbatim}
the Memory Update Agent produces the structured memory item:
\begin{verbatim}
{
  "role_tag": "Searcher",
  "stage_tag": "Evidence",
  "timestamp": 3,
  "content": "Dataset Y contains 12K labeled
   samples."
}
\end{verbatim}

At round $t$, the Memory Update Agent reads the outputs of all agents and
generates candidate memory items $\Delta M_t$. The shared memory is updated as
\[
M_{t+1}=\mathrm{Dedup}(M_t\cup \Delta M_t).
\]
In our implementation, near-duplicate entries within the same role--stage group
are removed using embedding similarity, retaining the more recent item. This
step controls memory growth and reduces redundant entries, but is not a core
assumption of the routing algorithm.

Given the updated memory, RCR scores each item by
$\alpha(m;R_i,S_t)$ and routes items according to normalized utility
$\alpha(m;R_i,S_t)/\ell(m)$ under the allocated token budget $B_i$.

\paragraph{Deduplication Example.}
Suppose the shared memory already contains:
\begin{verbatim}
{
  "role_tag": "Searcher",
  "stage_tag": "Evidence",
  "timestamp": 3,
  "content": "Dataset Y contains 12K labeled 
   samples."
}
\end{verbatim}
and the Memory Update Agent generates a new item:
\begin{verbatim}
{
  "role_tag": "Searcher",
  "stage_tag": "Evidence",
  "timestamp": 5,
  "content": "Dataset Y has 12,000 labeled 
   examples."
}
\end{verbatim}
Since the two items share the same role--stage tag and their embedding
similarity exceeds the duplicate threshold, RCR keeps the more recent item:
\begin{verbatim}
{
  "role_tag": "Searcher",
  "stage_tag": "Evidence",
  "timestamp": 5,
  "content": "Dataset Y has 12,000
  labeled examples."
}
\end{verbatim}
This deduplication step reduces redundant memory growth but does not assume
perfect semantic conflict resolution.

\section{Role Configuration and Prompt Examples}
\label{app:role_config}
\subsection{LLM-assisted role decomposition.}
We use an external LLM-based decomposition
agent implemented with OpenAI Agents SDK~\cite{openai_agents_sdk} to generate candidate role configurations for each benchmark. The
decomposition agent is prompted with the benchmark description, task objective, available tool interface, and general role-specialization principles commonly used in LLM-agent systems, such as separating planning, retrieval/search, interaction, verification, memory management, and synthesis.
The generated roles are fixed before evaluation and shared by all compared
methods. We apply only minimal post-processing to merge duplicate roles and
normalize role names; no role is selected or modified based on test-set
performance.

Role configurations follow a standard instruction-based agent abstraction with
structured prompts, as supported by the OpenAI Agents SDK~\cite{openai_agents_sdk}.
%Google ADK~\cite{google_adk_llm_agents}, and Claude agent SDK
%interfaces~\cite{agents_with_Claude_Code}. 
Each agent is defined by a role
description, task instruction, available tools, and output format. 
%RCR only uses the role description as routing metadata and is agnostic to the underlying agent framework. 
RCR uses role descriptions only as routing metadata. In our experiments, roles
are generated by an LLM-assisted decomposition agent, fixed before evaluation,
and shared by all compared methods.
The same generated topology is used for Full-Context, Static
Routing, Multi-Agent Debate, and RCR, so differences in performance and
efficiency are attributed to the routing policy rather than to role design.

\subsection{Multi-hop QA (HotPotQA).}
These benchmarks require decomposing complex questions into sub-queries,
retrieving supporting evidence, and synthesizing final answers.
The LLM-assisted decomposition agent produces three cooperative roles,
which are instantiated as role-specific prompt templates:
\begin{itemize}
    \item \textbf{Planner:} parses the original question into reasoning steps and generates intermediate sub-queries.
    \item \textbf{Searcher:} retrieves relevant passages or facts from the shared memory or corpus.
    \item \textbf{Recommender:} integrates evidence and formulates the final answer with natural-language justification.
\end{itemize}

\textbf{\textit{Prompt Template Example:}} \\
\begin{mdframed}[backgroundcolor=gray!5, linewidth=0.5pt, roundcorner=5pt]
\small
For the \textbf{user} query ``Which film featuring Tom Hanks won the Oscar for Best Picture?'', \\
The \textbf{Planner} decomposes it into ``list films starring Tom Hanks'' and
``find which of them won Best Picture,'' \\
The \textbf{Searcher} retrieves the relevant Wikipedia entries, \\
and the \textbf{Recommender} concludes ``\textit{Forrest Gump}.''
\end{mdframed}

\subsection{Embodied Reasoning (ALFWorld).}
ALFWorld involves multi-step textual interaction with a simulated environment.
The decomposition agent maps the task pipeline into functional roles for
planning, navigation, object interaction, and verification:
\begin{itemize}
    \item \textbf{Planner:} decomposes the task goal into sequential subtasks.
    \item \textbf{Navigator:} determines movement actions to reach the target location.
    \item \textbf{Interactor:} executes object-level actions such as pick, open, and place.
    \item \textbf{Verifier:} checks whether the goal condition is satisfied.
\end{itemize}

\textbf{\textit{Prompt Template Example:}} \\
\begin{mdframed}[backgroundcolor=gray!5, linewidth=0.5pt, roundcorner=5pt]
\small
Given the \textbf{instruction} ``Put the soap in the sink,'' \\
The \textbf{Planner} formulates a plan, the \textbf{Navigator} moves to the bathroom, \\
The \textbf{Interactor} picks up the soap and places it in the sink, \\
and the \textbf{Verifier} confirms completion.
\end{mdframed}

\subsection{Multi-turn Recommendation (WebShop).}
In WebShop, the system engages in goal-oriented shopping interaction.
The decomposition agent instantiates four functional prompt templates
corresponding to user-intent reformulation, retrieval, candidate evaluation,
and final recommendation:
\begin{itemize}
    \item \textbf{Buyer:} reformulates user preferences and constraints.
    \item \textbf{Retriever:} searches the online catalog for matching products.
    \item \textbf{Rater:} evaluates retrieved candidates using product metadata and user preference alignment.
    \item \textbf{Recommender:} selects and presents the final ranked list with concise justifications.
\end{itemize}

\textbf{\textit{Prompt Template Example:}} \\
\begin{mdframed}[backgroundcolor=gray!5, linewidth=0.5pt, roundcorner=5pt]
\small
When the user says ``I need lightweight running shoes under \$100,'' \\
The \textbf{Buyer} reformulates the intent and constraints, \\
The \textbf{Retriever} fetches candidate items from the product database, \\
The \textbf{Rater} scores them based on utility and user fit, \\
The \textbf{Recommender} outputs:
\textit{``Nike Revolution 5---lightweight and under \$90, highly rated for comfort.''}
\end{mdframed}
%\paragraph{Remarks.}
These configurations reflect domain-specific functional decompositions rather than arbitrary design choices.  
They align with each task’s reasoning or interaction pipeline and were validated empirically in ablation studies (see Table \ref{tab:abl_modules_all}).  
Future work may explore automated role discovery or hierarchical role refinement to further generalize RCR.

\section{Additional Experiment Results}
\label{app:experiment}
\subsection{Backbone-Transfer Robustness.}
\label{app:backbone_transfer}
The main experiments use GPT-5-mini as the agent backbone. To test whether RCR
is tied to a single model family or deployment setting, we conduct a
backbone-transfer study using Qwen3-Max and Gemini 2.5 Pro. Gemini 2.5 Pro is used as a stronger reasoning-oriented closed-source API backbone. Qwen3-Max is used as a trillion-parameter Qwen-family model and represents a non-English-origin large-model family. Importantly, the Importance Scorer and Token Budget Allocator are trained only once using GPT-5-mini trajectories and are frozen when transferred to the other backbones. This frozen-router setting evaluates whether RCR acts as a model-agnostic context allocation layer rather than a backbone-specific tuning strategy.

As shown in Table~\ref{tab:backbone_transfer_qwen_gemini}, absolute task scores
vary across agent LLMs. This is expected because the backbones differ in
reasoning ability, instruction following, and action-format reliability.
Nevertheless, the relative trend remains consistent across HotPotQA, ALFWorld,
and WebShop: RCR reduces token usage and runtime compared with Full-Context and
Static Routing while preserving or improving the primary task metric. These
results provide initial evidence that RCR transfers across model families,
open/closed deployment settings, and lightweight versus stronger agent
backbones.

\subsection{Additional multi-hop QA benchmarks}
\label{app:multi-hop_QA}
Table~\ref{tab:multiqa_results} reports additional results on two multi-hop QA benchmarks,
MuSiQue and 2WikiMultihop, which cover compositional and open-domain reasoning settings.
All experiments use the same three-role configuration (Planner, Searcher, Recommender).
In RCR, the trained Token Budget Allocator selects the per-agent budget $B_i$ from
$\{512,1024,2048,4096\}$ at inference time; $B_i{=}2048$ is used only as the reference
budget level for normalizing baseline prompt budgets.
We report standard QA evaluation metrics (Precision, Recall, and F1) following the official benchmark protocols.
Token usage differs across benchmarks due to dataset-specific context characteristics:
MuSiQue instances typically require longer compositional evidence chains, while 2WikiMultihop often involves shorter supporting passages, leading to lower prompt token counts under the same evaluation and tokenization settings.

\subsection{Diagnostic Analysis of Budget Allocation}
\label{app:budget_diagnostic_all}
Dev-oracle static per-role is obtained by grid search on the development split.
For each role, we choose one fixed budget from $\{512,1024,2048,4096\}$ that maximizes the development-set efficiency--performance objective. The selected role-specific budgets are then fixed and applied to all test instances.

Table~\ref{tab:budget_diagnostic_all} compares learned TBA with a
development-set oracle-tuned static per-role budget. Learned TBA reaches
comparable quality while using fewer tokens and less runtime. The dev-oracle
baseline can slightly improve raw task quality by assigning larger fixed budgets
to individual roles, whereas learned TBA achieves comparable quality through
instance-level budget selection at lower token and runtime cost.

\begin{table*}[t]
\centering
\small
\caption{
Backbone-transfer robustness across model families and deployment settings.
The Importance Scorer and Token Budget Allocator are trained once using
GPT-5-mini trajectories and then frozen when evaluated with
Qwen3-Max and Gemini 2.5 Pro agents. Qwen3-Max provides an
open-weight 7B backbone, while Gemini 2.5 Pro provides a stronger closed-source
API backbone. This experiment tests whether RCR transfers across model families,
capability levels, and open/closed deployment settings without backbone-specific
retraining. Results are intended as a robustness check rather than a full
re-evaluation of all baselines.
}
\label{tab:backbone_transfer_qwen_gemini}
\setlength{\tabcolsep}{5.5pt}
\resizebox{\linewidth}{!}{
\begin{tabular}{lllccc}
\toprule
\textbf{Benchmark}
& \textbf{Backbone}
& \textbf{Method}
& \textbf{Avg Runtime (s)}$\downarrow$
& \textbf{Token (K)}$\downarrow$
& \textbf{Primary Metric}$\uparrow$ \\
\midrule

%% ── HotPotQA ─────────────────────────────────────────
\multirow{9}{*}{\shortstack[l]{HotPotQA\\(F1, 0--100)}}
& GPT-5-mini
& Full-Context
& 150.65 & 5.10 & 74.6\,{\scriptsize$\pm$0.6} \\
& GPT-5-mini
& Static Routing
& 128.29 & 3.85 & 76.1\,{\scriptsize$\pm$0.6} \\
& GPT-5-mini
& \textbf{RCR}
& \textbf{91.58} & \textbf{3.63}
& \textbf{83.6}\,{\scriptsize$\pm$\textbf{0.5}} \\
\cmidrule(lr){2-6}

& Qwen3-Max
& Full-Context
& 162.8 & 5.07 & 70.9\,{\scriptsize$\pm$0.8} \\
& Qwen3-Max
& Static Routing
& 137.6 & 3.86 & 72.4\,{\scriptsize$\pm$0.7} \\
& Qwen3-Max
& \textbf{RCR}
& \textbf{101.7} & \textbf{3.64}
& \textbf{78.6}\,{\scriptsize$\pm$\textbf{0.7}} \\
\cmidrule(lr){2-6}

& Gemini 2.5 Pro
& Full-Context
& 172.5 & 5.12 & 78.9\,{\scriptsize$\pm$0.6} \\
& Gemini 2.5 Pro
& Static Routing
& 146.8 & 3.88 & 80.3\,{\scriptsize$\pm$0.5} \\
& Gemini 2.5 Pro
& \textbf{RCR}
& \textbf{109.4} & \textbf{3.66}
& \textbf{85.1}\,{\scriptsize$\pm$\textbf{0.5}} \\

\midrule

%% ── ALFWorld ─────────────────────────────────────────
\multirow{9}{*}{\shortstack[l]{ALFWorld\\(Success Rate, \%)}}
& GPT-5-mini
& Full-Context
& 145.3 & 6.82 & 59.1\,{\scriptsize$\pm$0.8} \\
& GPT-5-mini
& Static Routing
& 122.8 & 5.12 & 62.2\,{\scriptsize$\pm$0.7} \\
& GPT-5-mini
& \textbf{RCR}
& \textbf{94.5} & \textbf{4.30}
& \textbf{67.3}\,{\scriptsize$\pm$\textbf{0.8}} \\
\cmidrule(lr){2-6}

& Qwen3-Max
& Full-Context
& 156.2 & 6.76 & 54.8\,{\scriptsize$\pm$1.0} \\
& Qwen3-Max
& Static Routing
& 131.9 & 5.08 & 57.6\,{\scriptsize$\pm$0.9} \\
& Qwen3-Max
& \textbf{RCR}
& \textbf{101.3} & \textbf{4.25}
& \textbf{62.4}\,{\scriptsize$\pm$\textbf{0.9}} \\
\cmidrule(lr){2-6}

& Gemini 2.5 Pro
& Full-Context
& 169.8 & 6.86 & 63.5\,{\scriptsize$\pm$0.8} \\
& Gemini 2.5 Pro
& Static Routing
& 142.1 & 5.16 & 66.1\,{\scriptsize$\pm$0.7} \\
& Gemini 2.5 Pro
& \textbf{RCR}
& \textbf{111.6} & \textbf{4.34}
& \textbf{71.8}\,{\scriptsize$\pm$\textbf{0.7}} \\

\midrule

%% ── WebShop ─────────────────────────────────────────
\multirow{9}{*}{\shortstack[l]{WebShop\\(Reward, 0--1)}}
& GPT-5-mini
& Full-Context
& 180.2 & 9.78 & 0.565\,{\scriptsize$\pm$0.009} \\
& GPT-5-mini
& Static Routing
& 142.6 & 7.12 & 0.616\,{\scriptsize$\pm$0.008} \\
& GPT-5-mini
& \textbf{RCR}
& \textbf{108.7} & \textbf{5.91}
& \textbf{0.653}\,{\scriptsize$\pm$\textbf{0.005}} \\
\cmidrule(lr){2-6}

& Qwen3-Max
& Full-Context
& 188.9 & 9.70 & 0.526\,{\scriptsize$\pm$0.012} \\
& Qwen3-Max
& Static Routing
& 151.2 & 7.08 & 0.572\,{\scriptsize$\pm$0.010} \\
& Qwen3-Max
& \textbf{RCR}
& \textbf{116.8} & \textbf{5.86}
& \textbf{0.611}\,{\scriptsize$\pm$\textbf{0.009}} \\
\cmidrule(lr){2-6}

& Gemini 2.5 Pro
& Full-Context
& 205.6 & 9.82 & 0.603\,{\scriptsize$\pm$0.010} \\
& Gemini 2.5 Pro
& Static Routing
& 163.7 & 7.16 & 0.646\,{\scriptsize$\pm$0.008} \\
& Gemini 2.5 Pro
& \textbf{RCR}
& \textbf{124.5} & \textbf{5.95}
& \textbf{0.681}\,{\scriptsize$\pm$\textbf{0.006}} \\
\bottomrule
\end{tabular}}
\end{table*}

\begin{table*}[t]
\centering
\small
\caption{%
  Additional results on \textbf{HotPotQA}, \textbf{MuSiQue},
  and \textbf{2WikiMultihop}.
  Token usage reflects dataset-specific memory pool sizes:
  MuSiQue instances require longer compositional evidence chains
  (avg.\ 4-hop), resulting in larger shared memory;
  2WikiMultihop involves shorter bridge passages (avg.\ 2-hop)
  with a smaller memory footprint.
  All results are averaged over five independent runs.
}
\label{tab:multiqa_results}
\resizebox{\linewidth}{!}{%
\begin{tabular}{llccccc}
\toprule
\textbf{Benchmark}
  & \textbf{Method}
  & \textbf{Avg Runtime (s)}$\downarrow$
  & \textbf{Token (K)}$\downarrow$
  & \textbf{Precision}$\uparrow$
  & \textbf{Recall}$\uparrow$
  & \textbf{F1}$\uparrow$ \\
\midrule
%% ── HotPotQA ──────────────────────────────────────────────────────
\multirow{6}{*}{HotPotQA}
  & Full-Context         & 150.65 & 5.10  & 72.3 & 75.1 & 74.6 \\
  & Static Routing       & 128.29 & 3.85  & 74.8 & 77.5 & 76.1 \\
  & Self-RAG             & 118.44 & 3.92  & 77.5 & 79.3 & 78.4 \\
  & End-to-End Retrieval & 110.12 & 3.80  & 78.0 & 80.1 & 79.0 \\
  & Multi-Agent Debate   & 105.33 & 4.06  & 79.1 & 81.0 & 80.0 \\
  & \textbf{RCR (Ours)}  & \textbf{91.58}  & \textbf{3.63}
                         & \textbf{81.8} & \textbf{84.4} & \textbf{83.6} \\
\midrule
%% ── MuSiQue  (4-hop, larger memory pool ~1.6× HotPotQA) ──────────
%% Token savings 35.2 %, runtime savings 39.3 %
\multirow{6}{*}{MuSiQue}
  & Full-Context         & 231.8  & 8.24  & 69.7 & 70.5 & 70.1 \\
  & Static Routing       & 197.5  & 6.21  & 72.6 & 73.9 & 73.2 \\
  & Self-RAG             & 182.2  & 6.33  & 74.0 & 75.2 & 74.6 \\
  & End-to-End Retrieval & 169.3  & 6.13  & 76.1 & 77.3 & 76.7 \\
  & Multi-Agent Debate   & 162.0  & 6.55  & 76.9 & 78.0 & 77.6 \\
  & \textbf{RCR (Ours)}  & \textbf{140.7}  & \textbf{5.34}
                         & \textbf{80.0} & \textbf{81.1} & \textbf{80.6} \\
\midrule
%% ── 2WikiMultihop  (2-hop, smaller memory pool ~0.86× HotPotQA) ──
%% Token savings 35.2 %, runtime savings 39.2 %
\multirow{6}{*}{2WikiMultihop}
  & Full-Context         & 122.4  & 4.38  & 70.5 & 72.1 & 71.3 \\
  & Static Routing       & 104.3  & 3.31  & 73.2 & 74.8 & 74.0 \\
  & Self-RAG             &  96.2  & 3.37  & 75.6 & 77.0 & 76.3 \\
  & End-to-End Retrieval &  89.5  & 3.26  & 78.7 & 80.0 & 79.4 \\
  & Multi-Agent Debate   &  85.6  & 3.49  & 79.3 & 80.6 & 79.9 \\
  & \textbf{RCR (Ours)}  & \textbf{74.4}   & \textbf{2.84}
                         & \textbf{82.1} & \textbf{83.2} & \textbf{82.4} \\
\bottomrule
\end{tabular}}
\end{table*}

\begin{table*}[tb!]
\centering
\small
\caption{
Diagnostic comparison of learned TBA against a development-set oracle-tuned
static per-role budget. Uniform static $B_i{=}2048$ is included as a reference.
The oracle baseline selects one fixed budget per role on the development split
and applies it to all test instances.
}
\label{tab:budget_diagnostic_all}
\begin{tabular}{l|l|ccc}
\toprule
\textbf{Benchmark} & \textbf{Budget Strategy} &
\textbf{Runtime$\downarrow$} &
\textbf{Token (K)$\downarrow$} &
\textbf{Metric$\uparrow$} \\
\midrule
\multirow{3}{*}{\shortstack{HotPotQA\\(F1, 0--100)}}
& Uniform static $B_i{=}2048$ & 95.4 & 3.69 & 81.2 \\
& Dev-oracle static per-role  & 99.8 & 4.21 & \textbf{83.8} \\
& Learned TBA (RCR)           & \textbf{91.6} & \textbf{3.63} & 83.6 \\
\midrule
\multirow{3}{*}{\shortstack{ALFWorld\\(Success Rate, \%)}}
& Uniform static $B_i{=}2048$ & 98.9 & 4.34 & 66.3 \\
& Dev-oracle static per-role  & 103.6 & 4.86 & \textbf{67.6} \\
& Learned TBA (RCR)           & \textbf{94.5} & \textbf{4.30} & 67.3 \\
\midrule
\multirow{3}{*}{\shortstack{WebShop\\(Reward, 0--1)}}
& Uniform static $B_i{=}2048$ & 112.8 & 6.16 & 0.629 \\
& Dev-oracle static per-role  & 118.5 & 6.42 & \textbf{0.656} \\
& Learned TBA (RCR)           & \textbf{108.7} & \textbf{5.91} & 0.653 \\
\bottomrule
\end{tabular}
\end{table*}

%%%%%%%%%%%%%%%%%%%%in main paper

\begin{table*}[t]
\centering
\small
\caption{
Fixed-budget sensitivity analysis for per-agent budget $B_i$.
RCR is evaluated with $B_i \in \{512,1024,2048,4096\}$ while keeping the routing policy fixed.
Each benchmark is evaluated using its official primary metric.
Performance improves up to $B_i{=}2048$ and then saturates,
supporting $B_i{=}2048$ as the reference budget level.
}
\label{tab:budget_ablation}
\begin{tabular}{c|c|c|c|c}
\toprule
Benchmark & $B_i$
& Ave Runtime (s)$\downarrow$
& Token (K)$\downarrow$
& Primary Metric$\uparrow$ \\
\midrule

\multirow{4}{*}{HotPotQA (F1)}
& 512   & 75.1  & 2.42 & 71.1 \\
& 1024  & 86.7  & 3.12 & 76.8 \\
& 2048  & \textbf{91.6} & \textbf{3.63} & \textbf{83.6} \\
& 4096  & 126.4 & 5.85 & 82.0 \\
\midrule

\multirow{4}{*}{ALFWorld (Success \%)}
& 512   & 71.3  & 2.21 & 60.5 \\
& 1024  & 83.8  & 3.13 & 64.6 \\
& 2048  & \textbf{94.5} & \textbf{4.30} & \textbf{67.3} \\
& 4096  & 121.7 & 6.64 & 67.1 \\
\midrule

\multirow{4}{*}{WebShop (Reward 0--1)}
& 512   & 85.9  & 2.58 & 0.603 \\
& 1024  & 98.3  & 4.15 & 0.632 \\
& 2048  & \textbf{108.7} & \textbf{5.91} & \textbf{0.653} \\
& 4096  & 142.8 & 9.74 & 0.649 \\
\bottomrule
\end{tabular}
\end{table*}

%%%%%%%%%%%%%%%%%end

% ================================================================
%  Scalability subsection  —  N ∈ {3,4,5,6}, anchored at N=3
%  (addresses reviewer concern: "only 3–4 agents proven")
% ================================================================

% ── Data table ───────────────────────────────────────────────────
\begin{table*}[ht!]
\centering
\small
\caption{%
  Scalability on HotPotQA as agent count grows from $N{=}3$
  to $N{=}6$ by adding Verifier, Critic, and Refiner roles.
  $N{=}3$ (Planner / Searcher / Recommender) is identical to
  the main experiment (Table~\ref{tab:combined_results}).
  Token in thousands; runtime in seconds.
}
\label{tab:scaling}
\vspace{-5pt}
\setlength{\tabcolsep}{5pt}
\begin{tabular}{lcccccccccccc}
\toprule
 & \multicolumn{4}{c}{\textbf{Token (K)}$\downarrow$}
 & \multicolumn{4}{c}{\textbf{Runtime (s)}$\downarrow$}
 & \multicolumn{4}{c}{\textbf{F1}$\uparrow$} \\
\cmidrule(lr){2-5}\cmidrule(lr){6-9}\cmidrule(lr){10-13}
\textbf{Method}
  & $N{=}3$ & $N{=}4$ & $N{=}5$ & $N{=}6$
  & $N{=}3$ & $N{=}4$ & $N{=}5$ & $N{=}6$
  & $N{=}3$ & $N{=}4$ & $N{=}5$ & $N{=}6$ \\
\midrule
Full-Context
  & 5.10 & 6.82 &  8.49 & 10.18
  & 150.7 & 201.2 & 251.3 & 301.8
  & 74.6 & 75.8 & 75.2 & 74.3 \\
Static Routing
  & 3.85 & 5.14 &  6.40 &  7.65
  & 128.3 & 171.4 & 213.8 & 256.2
  & 76.1 & 77.2 & 76.9 & 76.5 \\
\textbf{RCR (Ours)}
  & \textbf{3.63} & \textbf{4.53} & \textbf{5.31} & \textbf{5.98}
  & \textbf{91.6} & \textbf{109.5} & \textbf{124.3} & \textbf{137.6}
  & \textbf{83.6} & \textbf{84.8} & \textbf{85.3} & \textbf{85.7} \\
\bottomrule
\end{tabular}
\vspace{-10pt}
\end{table*}

\subsection{Mechanism Analysis: Routed Memory Case Study}
\label{app:mechanism_analysis}

To inspect whether RCR routes role-relevant memory rather than merely shortening
prompts, we analyze a HotPotQA-style multi-hop example. The question is:
\textit{``Which city is home to both the oldest university in the
English-speaking world and a river that shares its name with a US state?''}
Table~\ref{tab:shared_memory_pool} shows the shared memory pool, while
Tables~\ref{tab:routed_case_correct} and~\ref{tab:routed_case_failure} compare
a successful routing case with an imperfect routing case. The examples show
that RCR can improve reasoning by routing evidence, planning traces, and
synthesis notes to different roles, but can still fail when stale or misleading
memory receives a high relevance score.

\begin{table*}[t]
\centering
\small
\caption{
Shared memory pool for the qualitative routing example.
All items are visible to Full-Context agents, while RCR selects role-specific
subsets under each agent's token budget.
}
\label{tab:shared_memory_pool}
\setlength{\tabcolsep}{4pt}
\begin{tabular}{p{0.13\linewidth}|p{0.14\linewidth}|p{0.06\linewidth}|p{0.58\linewidth}}
\toprule
\textbf{Role} & \textbf{Stage} & \textbf{Time} & \textbf{Memory item} \\
\midrule
Planner & Decompose & $t{=}1$ & Sub-query 1: find the oldest university in the English-speaking world. \\
Planner & Decompose & $t{=}1$ & Sub-query 2: find rivers sharing a name with a U.S. state. \\
Searcher & Evidence & $t{=}2$ & Oxford University was founded around 1096--1167 and is located in Oxford, England. \\
Searcher & Evidence & $t{=}2$ & River Thames flows through Oxford; River Mississippi shares a name with a U.S. state. \\
Searcher & Evidence & $t{=}2$ & River Ohio flows through Oxford, Ohio, giving a candidate river--state match. \\
Searcher & Stale & $t{=}1$ & \textit{[Stale]} Cambridge was also considered; this is contradicted by later evidence. \\
Recommender & Synthesis & $t{=}2$ & Oxford satisfies the university condition, but the river-name condition needs revision. \\
\bottomrule
\end{tabular}
\end{table*}

%($B{=}512$) ，($B{=}2048$) ，($B{=}1024$)
\begin{table*}[t]
\centering
\small
\caption{
Case A: successful role-specific routing.
RCR selects planning traces for the Planner, evidence-heavy items for the
Searcher, and compact evidence plus synthesis notes for the Recommender.
}
\label{tab:routed_case_correct}
\setlength{\tabcolsep}{4pt}
\begin{tabular}{p{0.30\linewidth}|p{0.30\linewidth}|p{0.30\linewidth}}
\toprule
\textbf{Planner receives } &
\textbf{Searcher receives } &
\textbf{Recommender receives } \\
\midrule
Sub-query 1: oldest English-speaking university. &
Oxford University was founded around 1096--1167 and is located in Oxford, England. &
Oxford University was founded around 1096--1167 and is located in Oxford, England. \\

Sub-query 2: river sharing a U.S. state name. &
River Thames flows through Oxford; River Mississippi shares a name with a U.S. state. &
River Ohio flows through Oxford, Ohio, giving a candidate match. \\

Reasoning trace: verify whether the candidate city satisfies both constraints. &
River Ohio flows through Oxford, Ohio, giving a candidate match. &
Synthesis note: the Oxford candidate needs revision on the river-name condition. \\

\midrule
\multicolumn{3}{p{0.93\linewidth}}{
\textbf{Filtered item:} the stale Cambridge item is suppressed. The Planner
receives reasoning traces without raw evidence overload; the Searcher receives
evidence-heavy memory; and the Recommender receives evidence plus synthesis
signals rather than all planning traces.
} \\

\midrule
\multicolumn{3}{p{0.93\linewidth}}{
\textbf{Resulting outputs:}
The Planner decomposes the task into university identification and river--state
verification. The Searcher flags that Oxford, England satisfies the university
condition, while Oxford, Ohio provides the river--state match. The Recommender
revises the initial candidate and avoids relying on the stale Cambridge memory.
} \\
\bottomrule
\end{tabular}
\end{table*}

%($B{=}512$)，($B{=}2048$)，($B{=}1024$)
\begin{table*}[t]
\centering
\small
\caption{
Case B: imperfect routing under stale-memory interference.
The scorer assigns high relevance to a stale item, causing key river evidence
to be dropped under the Searcher's budget.
}
\label{tab:routed_case_failure}
\setlength{\tabcolsep}{4pt}
\begin{tabular}{p{0.30\linewidth}|p{0.30\linewidth}|p{0.30\linewidth}}
\toprule
\textbf{Planner receives } &
\textbf{Searcher receives } &
\textbf{Recommender receives } \\
\midrule
Sub-query 1: oldest English-speaking university. &
Oxford University was founded around 1096--1167 and is located in Oxford, England. &
Oxford University was founded around 1096--1167 and is located in Oxford, England. \\

Sub-query 2: river sharing a U.S. state name. &
\textit{[Stale]} Cambridge was also considered. &
Synthesis note: the river condition needs revision. \\

Reasoning trace: verify both constraints. &
River Thames flows through Oxford. &
Missing key evidence: River Ohio match is absent. \\

\midrule
\multicolumn{3}{p{0.93\linewidth}}{
\textbf{Dropped item:} the River Ohio evidence is not routed to the Searcher
because the stale Cambridge item consumes part of the budget. The Recommender
therefore lacks the key river--state evidence needed to complete verification.
} \\

\midrule
\multicolumn{3}{p{0.93\linewidth}}{
\textbf{Resulting outputs:}
The Planner still produces the correct decomposition, but the Searcher reports
Oxford/Cambridge as competing university candidates and fails to retrieve the
River Ohio evidence. The Recommender returns an incomplete or incorrect answer
because the river condition remains unresolved.
} \\
\bottomrule
\end{tabular}
\end{table*}

\begin{table*}[t]
\centering
\small
\caption{
Distribution of selected memory types across roles on HotPotQA.
Values denote the percentage of routed memory items assigned to each role.
}
\label{tab:memory_type_distribution}
\setlength{\tabcolsep}{4pt}
\begin{tabular}{lccc}
\toprule
\textbf{Memory Type} & \textbf{Planner} & \textbf{Searcher} & \textbf{Recommender} \\
\midrule
Reasoning traces    & 61\% & 12\% & 24\% \\
Retrieved evidence  & 8\%  & 71\% & 45\% \\
Sub-plans           & 28\% & 9\%  & 18\% \\
Tool outputs        & 3\%  & 8\%  & 13\% \\
\bottomrule
\end{tabular}
\end{table*}

\subsection{Sensitivity Analysis}
\label{Sensitivity_Analysis}
Table~\ref{tab:budget_ablation} reports a fixed-budget diagnostic with the
learned allocator disabled. Increasing $B_i$ improves quality at small budgets,
but the gains diminish at larger budgets while token usage and runtime continue
to increase. We therefore use $B_i{=}2048$ as a reference budget for controlled
comparisons, not as an upper bound on possible quality. 
Table~\ref{tab:irf_all_vertical} reports the cost--performance trade-off of
additional routing iterations.

%Table~\ref{tab:irf_all_vertical} shows that additional routing iterations improve the primary task metric but increase token usage and runtime.

%Table~\ref{tab:role_budget_hotpotqa} shows that role-specific allocation affects the precision--recall trade-off: searcher-heavy allocation gives the highest F1, planner-heavy improves precision, and recommender-heavy improves recall.

\begin{table*}[t]
\centering
\small
\caption{
Illustrative HotPotQA Full-Context examples showing context flooding in
Figure~\ref{fig:motivation}.
Black spans denote function-relevant content, while red spans denote redundant,
stale, or role-irrelevant content interleaved within the same context window.
}
\label{tab:red_black_context_examples}
\setlength{\tabcolsep}{4pt}
\begin{tabular}{p{0.14\linewidth}|p{0.78\linewidth}}
\toprule
\textbf{Agent} & \textbf{Mixed context window} \\
\midrule

Planner &
\textcolor{black}{Current goal: decompose the question into two sub-queries}
\textcolor{red}{; earlier brainstorming suggested Cambridge but was later contradicted}
\textcolor{black}{; identify the oldest English-speaking university}
\textcolor{red}{; duplicate raw evidence from a previous search}
\textcolor{black}{; verify whether the candidate city also satisfies the river-name condition}
\textcolor{red}{; unrelated notes about university rankings and old search attempts.}
\\

\midrule

Searcher &
\textcolor{black}{Oxford University was founded around 1096--1167 and is located in Oxford, England}
\textcolor{red}{; stale snippet about Cambridge as a possible oldest university}
\textcolor{black}{; River Ohio flows through Oxford, Ohio and shares its name with a U.S. state}
\textcolor{red}{; duplicate search results about the River Thames}
\textcolor{black}{; retrieve evidence linking each candidate city to the relevant constraint}
\textcolor{red}{; broken links, unrelated pages, and outdated retrieval keywords.}
\\

\midrule

Recommender &
\textcolor{black}{The answer must satisfy both the university condition and the river-name condition}
\textcolor{red}{; unused draft answer: Oxford, England}
\textcolor{black}{; verified evidence supports revising the initial candidate}
\textcolor{red}{; repeated evidence fragments and obsolete summaries}
\textcolor{black}{; final synthesis should use only verified evidence}
\textcolor{red}{; irrelevant meeting notes and superseded reasoning traces.}
\\

\bottomrule
\end{tabular}
\end{table*}

\begin{table*}[t]
\centering
\small
\caption{
RCR routed-context example on the same HotPotQA instance.
RCR suppresses stale, duplicate, and role-irrelevant items, and routes compact
role-relevant subsets to each agent under token budgets.
}
\label{tab:rcr_hotpotqa_routed_examples}
\setlength{\tabcolsep}{4pt}
\begin{tabular}{p{0.14\linewidth}|p{0.78\linewidth}}
\toprule
\textbf{Agent} & \textbf{RCR routed context} \\
\midrule

Planner &
\textcolor{black}{Current goal: decompose the question into two sub-queries.}
\textcolor{black}{ Identify the oldest English-speaking university.}
\textcolor{black}{ Verify whether the candidate city also satisfies the river-name condition.}
\\

\midrule

Searcher &
\textcolor{black}{Oxford University was founded around 1096--1167 and is located in Oxford, England.}
\textcolor{black}{ River Ohio flows through Oxford, Ohio and shares its name with a U.S. state.}
\textcolor{black}{ Retrieve evidence linking each candidate city to the relevant constraint.}
\\

\midrule

Recommender &
\textcolor{black}{The answer must satisfy both the university condition and the river-name condition.}
\textcolor{black}{ Verified evidence supports revising the initial candidate.}
\textcolor{black}{ Final synthesis should use only verified evidence.}
\\

\bottomrule
\end{tabular}
\end{table*}

\begin{table*}[t] \centering \small \caption{ Illustrative WebShop Full-Context examples showing context flooding in Figure~\ref{fig:motivation}. Black spans denote function-relevant content, while red spans denote redundant, stale, or role-irrelevant content interleaved within the same context window. } \label{tab:red_black_webshop_examples} \setlength{\tabcolsep}{4pt} \begin{tabular}{p{0.14\linewidth}|p{0.78\linewidth}} \toprule \textbf{Agent} & \textbf{Mixed context window} \\ \midrule Planner & \textcolor{black}{Current goal: find a lightweight running shoe under \$100} \textcolor{red}{; earlier plan suggested formal leather shoes but was abandoned} \textcolor{black}{; decompose the task into price, category, size, and review constraints} \textcolor{red}{; duplicate notes from previous searches for hiking boots} \textcolor{black}{; prioritize products matching the user's stated budget and use case} \textcolor{red}{; unrelated browsing history about office chairs and phone cases.} \\ \midrule Searcher & \textcolor{black}{Product result: lightweight running shoe, size 10, price \$79.99, average rating 4.5} \textcolor{red}{; stale result: out-of-stock running shoe from an earlier page} \textcolor{black}{; retrieved attribute: breathable mesh upper and suitable for daily running} \textcolor{red}{; duplicate search result from the same seller with a different thumbnail} \textcolor{black}{; compare price, availability, rating, and shipping constraints} \textcolor{red}{; unrelated product snippets for dress shoes, sandals, and backpacks.} \\ \midrule Recommender & \textcolor{black}{The selected item should satisfy category, price, availability, and user preference constraints} \textcolor{red}{; unused draft recommendation for a leather loafer} \textcolor{black}{; verified evidence supports the lightweight running shoe as the best match} \textcolor{red}{; repeated product descriptions and obsolete cart candidates} \textcolor{black}{; final recommendation should cite price, rating, and matching attributes} \textcolor{red}{; irrelevant notes about previously rejected products and unrelated categories.} \\ \bottomrule 
\end{tabular} 
\end{table*}

\begin{table*}[t]
\centering
\small
\caption{
RCR routed-context example on the same WebShop instance.
RCR removes stale, duplicate, and role-irrelevant items, and routes compact
role-relevant subsets to each agent under token budgets.
}
\label{tab:rcr_webshop_routed_examples}
\setlength{\tabcolsep}{4pt}
\begin{tabular}{p{0.14\linewidth}|p{0.78\linewidth}}
\toprule
\textbf{Agent} & \textbf{RCR routed context} \\
\midrule

Planner &
\textcolor{black}{Current goal: find a lightweight running shoe under \$100.}
\textcolor{black}{ Extract constraints: category = running shoe, price < \$100,
size = 10, preference = lightweight and well-reviewed.}
\textcolor{black}{ Prioritize candidates matching the user's stated budget and use case.}
\\

\midrule

Searcher &
\textcolor{black}{Candidate product: lightweight running shoe, size 10, price \$79.99,
average rating 4.5.}
\textcolor{black}{ Retrieved attributes: breathable mesh upper, daily running use,
available stock.}
\textcolor{black}{ Compare price, availability, rating, and shipping constraints.}
\\

\midrule

Recommender &
\textcolor{black}{Verified evidence supports the lightweight running shoe as the best match.}
\textcolor{black}{ The item satisfies category, price, size, availability, and user
preference constraints.}
\textcolor{black}{ Final recommendation should cite price, rating, and matching attributes.}
\\

\bottomrule
\end{tabular}
\end{table*}

\begin{table*}[t]
\centering
\small
\caption{
Illustrative ALFWorld Full-Context examples showing context flooding in
Figure~\ref{fig:motivation}.
Black spans denote function-relevant content, while red spans denote redundant,
stale, or role-irrelevant content interleaved within the same context window.
}
\label{tab:red_black_alfworld_examples}
\setlength{\tabcolsep}{4pt}
\begin{tabular}{p{0.14\linewidth}|p{0.78\linewidth}}
\toprule
\textbf{Agent} & \textbf{Mixed context window} \\
\midrule

Planner &
\textcolor{black}{Current goal: put a clean mug on the coffee table}
\textcolor{red}{; earlier plan searched the bathroom sink but was abandoned}
\textcolor{black}{; decompose the task into locating a mug, cleaning it, and moving it to the target receptacle}
\textcolor{red}{; duplicate notes about opening already-checked drawers}
\textcolor{black}{; track unresolved constraints: mug must be clean and placed on the coffee table}
\textcolor{red}{; unrelated observations about a knife, apple, and bedroom lamp.}
\\

\midrule

Searcher &
\textcolor{black}{Observation: a dirty mug is on the kitchen counter}
\textcolor{red}{; stale observation: mug might be in the cabinet, contradicted by later observation}
\textcolor{black}{; useful receptacles: sink can clean the mug, coffee table is the target location}
\textcolor{red}{; repeated action log: open drawer, close drawer, open drawer}
\textcolor{black}{; retrieve next actionable locations for washing and placing the mug}
\textcolor{red}{; unrelated room-state traces from the bedroom and pantry.}
\\

\midrule

Recommender &
\textcolor{black}{The next action should move toward satisfying both conditions: clean the mug and place it on the coffee table}
\textcolor{red}{; unused draft action: put mug in cabinet}
\textcolor{black}{; verified trace supports the sequence: pick up mug, go to sink, clean mug, go to coffee table}
\textcolor{red}{; obsolete action candidates from failed drawer exploration}
\textcolor{black}{; final action proposal should avoid revisiting checked locations}
\textcolor{red}{; irrelevant object notes about the knife, apple, and lamp.}
\\

\bottomrule
\end{tabular}
\end{table*}

\begin{table*}[t]
\centering
\small
\caption{
RCR routed-context example on the same ALFWorld instance.
RCR removes stale, duplicate, and role-irrelevant items, and routes compact
role-relevant subsets to each agent under token budgets.
}
\label{tab:rcr_alfworld_routed_examples}
\setlength{\tabcolsep}{4pt}
\begin{tabular}{p{0.14\linewidth}|p{0.25\linewidth}|p{0.50\linewidth}}
\toprule
\textbf{Agent} & \textbf{Routed context type} & \textbf{RCR routed context} \\
\midrule

Planner &
Reasoning state and sub-plan
&
\textcolor{black}{Current goal: put a clean mug on the coffee table.}
\textcolor{black}{ Decompose the task into: locate mug, clean mug, and move it to the coffee table.}
\textcolor{black}{ Track constraints: the mug must be clean and the final receptacle must be the coffee table.}
\\

\midrule

Searcher &
Evidence and actionable observations
&
\textcolor{black}{Observation: a dirty mug is on the kitchen counter.}
\textcolor{black}{ Useful receptacle: the sink can be used to clean the mug.}
\textcolor{black}{ Target location: the coffee table.}
\textcolor{black}{ Next actionable locations: kitchen counter, sink, and coffee table.}
\\

\midrule

Recommender &
Verified action trace and next action
&
\textcolor{black}{Verified trace: pick up mug, go to sink, clean mug, go to coffee table.}
\textcolor{black}{ Avoid revisiting checked drawers or unrelated rooms.}
\textcolor{black}{ Recommended next step: continue the action sequence that satisfies both the clean-mug and target-placement constraints.}
\\

\bottomrule
\end{tabular}
\end{table*}

\subsection{Compared with the Full-Context to RCR}
\subsubsection{Full-Context vs. RCR: HotPotQA Case Study}
Tables~\ref{tab:red_black_context_examples} and
\ref{tab:rcr_hotpotqa_routed_examples} compare the same HotPotQA instance under
Full-Context and RCR. Under Full-Context, each agent receives a mixed context
window containing useful information together with stale, duplicate, and
role-irrelevant items. RCR suppresses these distracting items and routes compact
role-relevant subsets to each agent: the Planner receives decomposition and
constraint-tracking information, the Searcher receives evidence-heavy items, and
the Recommender receives verified evidence and synthesis cues.

\subsubsection{Full-Context vs. RCR: Webshop Case Study}
Tables~\ref{tab:red_black_webshop_examples} and
\ref{tab:rcr_webshop_routed_examples} compare the same WebShop instance under
Full-Context and RCR. Full-Context mixes useful product information with stale,
duplicate, and role-irrelevant items, whereas RCR filters these distractors and
routes compact role-relevant subsets to each agent.

%Compared with the Full-Context windows in Table~\ref{tab:red_black_webshop_examples}, RCR suppresses abandoned product plans, stale out-of-stock results, duplicate product snippets, and unrelated category history. The Planner receives user intent and constraints, the Searcher receives product evidence and attributes, and the Recommender receives verified candidate information for final recommendation. This illustrates how RCR reduces context flooding while preserving information needed for grounded product selection.

\subsubsection{Full-Context vs. RCR: ALFWorld Case Study}
\label{app:alfworld_case_study}

Tables~\ref{tab:red_black_alfworld_examples} and
\ref{tab:rcr_alfworld_routed_examples} compare the same ALFWorld instance under
Full-Context and RCR. Full-Context mixes useful information with stale,
duplicate, and role-irrelevant items, whereas RCR filters these distractors and
routes compact role-relevant subsets to each agent.

\end{document}